\documentclass{selfevolagent}

\usepackage{amsmath}
\usepackage{amssymb}
\usepackage{tabularx,array}
\usepackage{dsfont}

\usepackage{algorithm}
\usepackage{algorithmic}

\usepackage{amsmath}
\usepackage{amssymb}
\usepackage{mathtools}
\usepackage{amsthm}

\usepackage{comment}

\theoremstyle{plain}
\newtheorem{theorem}{Theorem}[section]

\newtheorem{lemma}[theorem]{Lemma}

\theoremstyle{definition}
\newtheorem{definition}[theorem]{Definition}

\theoremstyle{remark}
\newtheorem{remark}[theorem]{Remark}

\usepackage[textsize=tiny]{todonotes}

\newcommand{\E}{\mathbb{E}}

\newcommand{\zijin}[1]{{\textcolor{blue} {#1}}}

\begin{document}

\title{Stabilizing Reinforcement Learning for Diffusion Language Models}
\author[1,2]{Jianyuan Zhong\cofirst}
\author[1,3]{Kaibo Wang\cofirst}
\author[1,3]{Ding Ding\cofirst}
\author[1]{Zijin Feng\corrauthor}
\author[1]{Haoli Bai}
\author[3]{Yang Xiang}
\author[1]{Jiacheng Sun\corrauthor}
\author[2]{Qiang Xu\corrauthor}
\affiliation[1]{Huawei Foundation Model Department}
\affiliation[2]{The Chinese University of Hong Kong}
\affiliation[3]{The Hong Kong University of Science and Technology}
\code{\url{https://github.com/JianyuanZhong/StableDRL}}

\abstract{%
Group Relative Policy Optimization (GRPO) is highly effective for post-training autoregressive (AR) language models, yet its direct application to diffusion large language models (dLLMs) often triggers reward collapse. We identify two sources of incompatibility. First, GRPO relies on importance ratios defined by sequence probabilities, which are intractable in dLLMs and must be estimated (e.g., via ELBO-based or mean-field likelihood proxies), yielding inherently noisy ratios. Second, standard GRPO's formulation is not designed for estimated ratios: its conditional clipping can be anomalously bypassed by model-agnostic estimation noise, producing gradient spikes, while its fixed group-size normalization amplifies gradient-magnitude fluctuations under high-variance ratio estimates. We show these effects form a self-reinforcing instability loop that drives policy drift and further increases ratio variance. To break this loop, we propose StableDRL, a reformulation of GRPO tailored for dLLMs that uses (i) unconditional clipping to suppress outlier-induced spikes and (ii) self-normalization to constrain updates within the convex hull of per-sample gradients. We further extend StableDRL to block-wise diffusion models via a staircase attention mechanism.%
}

\maketitle

\begingroup
\renewcommand{\thefootnote}{\fnsymbol{footnote}}
\footnotetext[2]{Co-first authors.}
\footnotetext[3]{Corresponding authors.}
\endgroup

\begin{figure}[t]
  \centering
  \includegraphics[width=\linewidth]{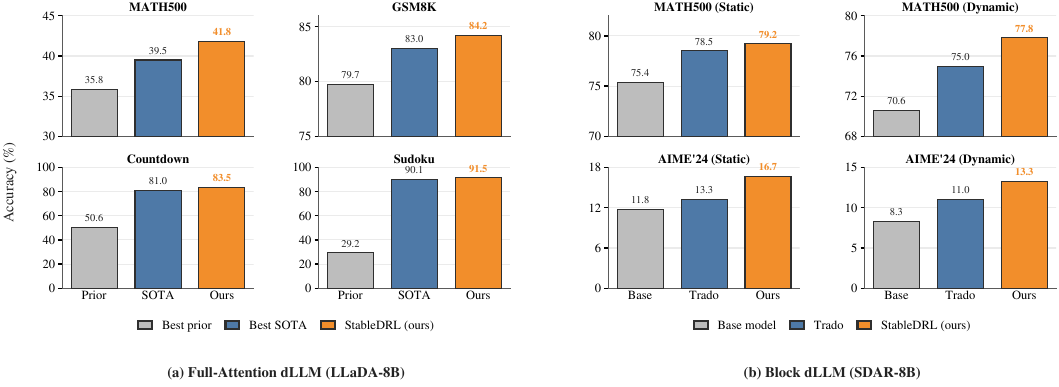}
  \caption{\textbf{StableDRL} is the first method to enable stable full-parameter RL training on both full-attention and block dLLMs, better unlocking reasoning capability for dLLMs. The \textbf{left} panel reports performance on full-attention dLLMs (LLaDA-8B~\cite{nie2025llada}). Based on Table~1, \emph{Best Prior} corresponds to WD1~\cite{tang2025wd1}, and \emph{Best SOTA} corresponds to the best performance between ESPO and SPG~\cite{ou2025espo, wang2025spg} for each task. The \textbf{right} panel demonstrates results for block diffusion models (SDAR-8B~\cite{cheng2025sdar}).}
  \label{fig:teaser}
\end{figure}

\section{Introduction}
\label{sec:intro}

Discrete Diffusion Large Language Models (dLLMs) have emerged as a compelling alternative to autoregressive (AR) models, intrinsically supporting parallel decoding and bidirectional context modeling~\citep{sahoo2024simple,nie2025llada,wu2025fast,yang2025mmada}. While Group Relative Policy Optimization (GRPO) has proven highly effective for reinforcement learning (RL) in the AR paradigm, its direct application to dLLMs leads to severe instability. As shown in Figure~\ref{fig:overview}(a), full-parameter GRPO training  on dLLMs exhibits an abrupt reward collapse at $\sim 300$ steps.

The incompatibility between GRPO and dLLMs stems from two factors: (i) the intractability of importance ratios in dLLMs~\cite{ou2025absorbingdiscretediffusionsecretly, ou2025espo} and (ii) GRPO's lack of adaptation to estimated importance ratios (Section ~\ref{sec:loop_analysis}). GRPO updates a target policy using data sampled from a behavior policy based on the importance ratios, defined as the ratio of their sequence probabilities. While this probability is tractable for AR models, it is intractable for dLLMs and often computed via estimations. Prior research has focused on the dLLM aspect, refining importance ratio estimation using mean-field approximations~\citep{zhao2025d1,tang2025wd1} or Evidence Lower Bound (ELBO) estimations~\citep{yang2025mmada,wang2025spg,ou2025espo}. Although these approaches yield performance gains, they empirically remain prone to training instability.

We attribute the instability in dLLMs to two design flaws in standard GRPO, which is inherently sensitive to the noisy importance ratios. First, the clipping mechanism in GRPO is conditional. In AR models, this mechanism accelerates the policy's return to the trust region. In dLLMs, however, model-agnostic estimation noise allows the clipping condition to be anomalously bypassed, triggering gradient spikes. Second, GRPO normalizes updates by the fixed group size. Given the high variance of importance ratio estimation in dLLMs, this static normalization results in drastic fluctuations in gradient magnitude, destabilizing the optimization process. To address the instability, we first analyze the underlying mechanism and then propose a stable GRPO variant tailored for dLLMs.

We theoretically and empirically demonstrate how these flaws precipitate a self-reinforcing \textit{instability loop}, leading to the reward collapse. As shown in Figure~\ref{fig:overview}(b), noisy importance ratios first induce gradient spikes under the GRPO update (Link 1). These spikes degrade the target policy, causing it to deviate significantly from the behavior policy (Link 2). This deviation, in turn, exacerbates the variance of importance ratios in subsequent steps (Link 3). We have proven that once the gradient norm exceeds a critical threshold, the probability of continued divergence increases, driving the policy toward irreversible reward collapse.

To stabilize training, we propose \textbf{StableDRL} to break the instability loop at its source (Link 1). As illustrated in Figure~\ref{fig:overview}(c), StableDRL incorporates two components. (i) We introduce \textit{unconditional clipping}, which enforces strict bounds on importance ratios regardless of the advantage. This prevents outliers from generating gradient spikes. (ii) We employ \textit{self-normalization}. Instead of dividing by the group size, we normalize the update by the sum of clipped importance ratios. This constrains the update within the convex hull of per-sample gradients. Furthermore, we extend StableDRL to block diffusion models~\citep{cheng2025sdar} via a \textit{staircase attention} mechanism, enabling leakage-free probability estimation.

To the best of our knowledge, StableDRL is the first method to enable stable, full-parameter RL training on both full-attention and block dLLMs for over 1,000 steps. This sustained stability effectively increases the volume of valid training rollout samples, allowing the model to fully unlock its reasoning capabilities and empirically achieve state-of-the-art performance in dLLM reasoning tasks. Our contributions are summarized as follows:
\begin{itemize}
\item We theoretically and empirically identify the self-reinforcing instability loop that causes reward collapse when GRPO is applied to dLLMs.
\item We propose StableDRL, a novel reinforcement learning framework for stabilizing the full-parameter training of dLLMs through unconditional clipping and self-normalization.
\item Comprehensive experiments validate the effectiveness of our StableDRL on both full-attention and block dLLMs, showing higher training stability and significant accuracy gain over prior best-in-class methods. 
\end{itemize}

\section{Background}
\begin{figure*}[t]
    \centering
    \includegraphics[width=1.0\linewidth]{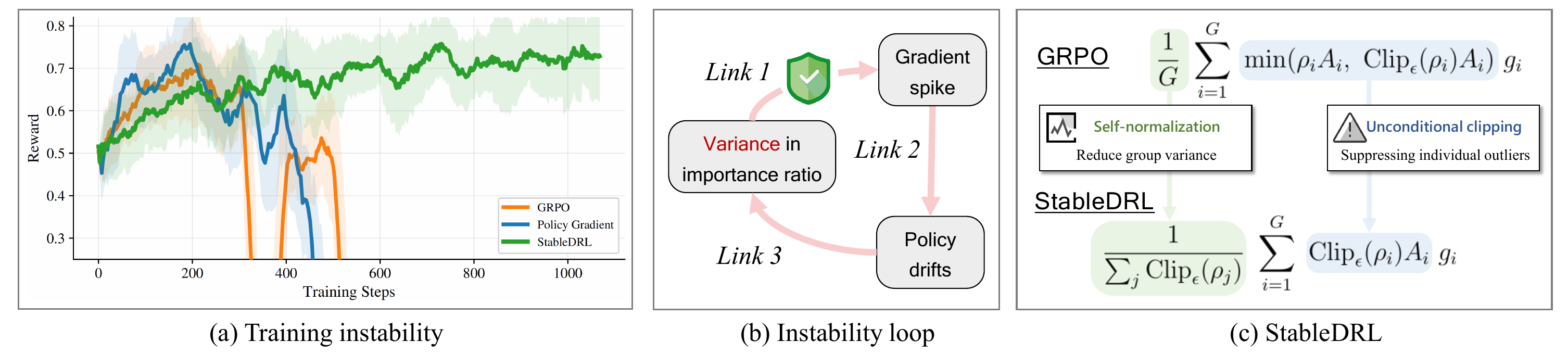}
    \caption{\textbf{(a)} \textbf{Training instability.} Naive integration of noisy importance ratios into GRPO leads to severe instability under full-parameter RL training with dLLMs. Notably, reward collapse occurs even with Policy Gradient, where the importance ratio is fixed at 1. \textbf{(b)} \textbf{Instability loop.} Estimation noise triggers gradient spikes and policy drift, creating a self-reinforcing cycle that amplifies the variance of future importance ratios. \textbf{(c)} \textbf{StableDRL.} To address this, we propose a reformulated GRPO for noisy importance ratios. By employing unconditional clipping and self-normalization, StableDRL effectively breaks the instability loop.}
    \label{fig:overview}
\end{figure*}
\label{sec:background}

\subsection{Masked Diffusion Language Models}
\label{subsec:mdlm}

Masked diffusion language models (MDLMs)~\citep{nie2025llada, sahoo2024simpleeffectivemaskeddiffusion} formulate text generation as a discrete diffusion process modeled by a continuous-time Markov chain. 
Given a clean sequence $x_0 \in \mathcal{V}^n$, the forward process $q(x_t|x_0)$ independently corrupts tokens by transitioning them to a special mask token \texttt{M} according to a schedule $t \in [0,1]$. The generative process reverses this corruption by learning a denoiser $\pi_\theta(x_0|x_t)$ to reconstruct the original data from the latent state.

Unlike AR models where the exact sequence log-likelihood $\log \pi_{\theta}(x_0)$ is computationally tractable via the chain rule, that of MDLMs is intractable, as it requires marginalizing over all $n!$ masking trajectories~\cite{ou2025absorbingdiscretediffusionsecretly}. Consequently, training optimizes the Evidence Lower Bound (ELBO)~\citep{wu2025fast,DBLP:conf/iclr/OuNXZSLL25}, 
denoted as $\mathcal{L}_\theta(x_0)$  with $\mathcal{L}_\theta(x_0) \leq \log \pi_{\theta}(x_0)$:
\begin{equation}\label{equ:elbo_objective}
    \mathcal{L}_\theta(x_0) = \E_{t,x_{t}}\left[\frac{1}{t}\sum_{i=1}^{n} \mathds{1}(x^i_{t}=\texttt{M})\log \pi_{\theta}(x^{i}_0|x_{t})\right].
\end{equation}

In practice, this expectation is approximated via Monte Carlo (MC) sampling. We denote the single-sample MC estimator of the ELBO as $\hat{\mathcal{L}}_\theta(x_0)$.
\subsection{Reinforcement Learning with dLLMs}

We focus on fine-tuning dLLMs to maximize a reward function $R(x)$ using policy gradient methods. The standard objective is to maximize the expected return $\mathcal{J}(\theta) = \mathbb{E}_{x \sim \pi_\theta}[R(x)]$. Modern on-policy algorithms, such as PPO~\citep{schulman2017ppo} and GRPO~\citep{shao2024deepseekmath}, improve sample efficiency by utilizing importance sampling to update the policy using trajectories collected from a behavior policy $\pi_{\theta_\text{old}}$.

\textbf{Group Relative Policy Optimization (GRPO).} GRPO eliminates the value function critic by estimating the baseline using the group average of rewards. For a group of rollouts $\{x_1, \dots, x_G\}$ sampled from $\pi_{\theta_\text{old}}$ conditioned on prompt $c$, the gradient update formula is:
\begin{equation} \label{eq:espo_objective}
\nabla_\theta\mathcal{J}_{\text{GRPO}} = \mathbb{E}\left[ \frac{1}{G}\sum_{j=1}^{G} \min\Big( \rho_{j} A_{j}, \text{clip}_\epsilon(\rho_{j}) A_{j} \Big)g_j \right],
\end{equation}
where $A_j$ is the advantage standardized within the group, $\rho_j(x) = \frac{\pi_\theta(x)}{\pi_{\theta_\text{old}}(x)}$ is the importance ratio, and $\text{clip}(\cdot)$ ensure a trused region of $[1-\epsilon, 1+\epsilon]$. For simplicity, We omit dependencies on $x, x \sim \pi_{old}$ and denote the gradient $\nabla_\theta \log \pi_\theta(x)$ by $g$. GRPO retains the min-clip operation from PPO, implicitly regularizing the divergence between the target policy and the behavior policy. When $\rho$ tends to move away from the trust region (e.g., $\rho > 1+\epsilon, A>0$), the clip operation limits the update step size. When it tends to return to the trust region ($\rho > 1+\epsilon, A<0$), employing unclipped step sizes to accelerate training. 

\paragraph{Challenges in adapting GRPO to dLLMs.} Prior works primarily improve importance ratio estimation methods and directly porting them to Eq.~\eqref{eq:espo_objective}. While earlier methods like D1 and WD1~\citep{zhao2025d1,tang2025wd1} attempted a one-step mean-field approximation, \citet{zhao2025diffpo} showed this to be inaccurate. As a result, current state-of-the-art approaches~\citep{yang2025mmada,wang2025spg,ou2025espo} use multi-step Monte Carlo sampling to estimate likelihood via the Evidence Lower Bound (ELBO). However, in practical on-policy RL with limited MC steps ($m \le 5$)~\cite{ou2025espo, wang2025spg}, the estimation suffers from noise and outliers. In Sec.~\ref{sec:loop_analysis}, we show that the combination of this estimation noise and the standard GRPO formulation causes training instability.

\section{Methodology}
\label{sec:method}

In this section, we first diagnose the root cause of the observed reward collapse in GRPO, identifying an \textit{instability loop} driven by the long-tail noise of importance ratios (Sec.~\ref{sec:loop_analysis}). We then propose \textbf{StableDRL}, which mitigates this instability through unconditional clipping and self-normalization (Sec.~\ref{sec:stabledrl}). Finally, we provide a theoretical justification for our method (Sec.~\ref{sec:theory}).

\subsection{Understanding Instability in dLLM RL Training}
\label{sec:loop_analysis}

As current state-of-the-arts utilize Monte Carlo sampling to estimate the intractable importance ratios of dLLMs,
we model the training instability as a three-stage process: (i) noise in estimated importance ratios forms a long-tail distribution; (ii) high variance and outliers generate gradient spikes; and (iii) these spikes induce policy drift, which amplifies future variances of the estimated importance ratios, closing the instability loop.

\paragraph{(i) Variance in importance ratios.}
Let $\eta(x)$ denote noise from the estimation error of the ELBO, such that $\hat{\mathcal{L}}_{\theta}(x)=\mathcal{L}_{\theta}(x)+\eta(x)$. The estimated $\hat{\rho}(x)$ can be decomposed into a policy drift term and a noise term:
\begin{equation}
\label{eq:noise_decomp}
\hat{\rho}(x) = \frac{\exp \hat{\mathcal{L}}_{\theta}(x)}{\exp \hat{\mathcal{L}}_{\theta_\text{old}}(x)} = \underbrace{\exp(\Delta \mathcal{L}(x))}_{\text{Policy Drift}} \cdot \underbrace{\exp(\Delta \eta(x))}_{\text{Noise}},
\end{equation}
where $\Delta \mathcal{L}(x)=\mathcal{L}_\theta(x) - \mathcal{L}_{\theta_{\text{old}}}(x)$ represents the true divergence between the target and behavior policies, and $\Delta \eta(x) =\eta_\theta(x) - \eta_{\theta_\text{old}}(x)$ stands for the net difference in estimation error between the two policy evaluations.

The $\exp(\cdot)$ operator maps the symmetric noise $\Delta \eta(x)$ to a long-tailed distribution with a non-negligible probability of yielding extreme values. For instance, the estimated ratio of a single rollout can explode to magnitudes of $10^5$ (Fig. ~\ref{fig:gradient_instability}).
Consequently, for a group of rollouts $\{x_1, \dots, x_G\}$, the resulting set of importance ratios $\{\hat{\rho}(x_1), \dots, \hat{\rho}(x_G)\}$ exhibits extremely high variance.

\paragraph{(ii) Gradient spikes.}
We observe that the noise in $\hat{\rho}(x)$ leads to gradient spikes through two mechanisms: individual anomalies and group anomalies.

\textit{Individual Anomalies.} In algorithms like GRPO, clipping is conditional. Specifically, when the advantage is negative ($A < 0$) and the ratio deviates significantly ($\hat{\rho} > 1+\epsilon$), the objective function simplifies to the unclipped term $\hat{\rho} A$. This design allows the model to take large steps when returning to the trust region. However, in dLLMs, large $\hat{\rho}$ values can be driven by the model-agnostic noise $\Delta \eta(x)$ rather than true policy alignment. Consequently, whenever $A<0$, there is a probability that a noise-induced outlier results in a massive, unclipped gradient.

\textit{Group Anomalies.} Due to the high variance of the estimator, importance ratios within a group $\{\hat{\rho}_j\}_{j=1}^G$ can be simultaneously large or small. Even if individual ratios are capped, the collective fluctuation of the sum $\sum \hat{\rho}_j$ causes the gradient magnitude to oscillate. In Sec.~\ref{sec:theory} and Sec.~\ref{sec:exp_instable}, we theoretically and empirically show that these frequent spikes, even bounded, can destabilize the training dynamics.

\paragraph{(iii) Policy drifts.}
When the target policy $\pi_{\theta}$ undergoes an update driven by a gradient spike, its behavior shifts abruptly, causing the policy divergence $\Delta \mathcal{L}(x)$ to increase significantly. As shown in Eq.~\eqref{eq:noise_decomp}, a larger $\Delta \mathcal{L}(x)$ acts as a multiplier, amplifying the variance of the importance ratios $\{\hat{\rho}_j\}_{j=1}^G$ in subsequent steps. This establishes an \textit{instability loop}: estimation noise generates gradient spikes, which induce policy drift; this drift, in turn, exacerbates the variance of future importance ratios. This self-reinforcing loop destabilizes training and leads to reward collapse.

\subsection{StableDRL}
\label{sec:stabledrl}

To stabilize training, we propose StableDRL. Our method breaks the instability loop by preventing importance ratio noise from translating into gradient spikes. It consists of two components: unconditional clipping and self-normalization.

\paragraph{Unconditional clipping.}
We replace the conditional clipping of GRPO with a strict, unconditional constraint. We enforce that the importance ratio $\hat{\rho}$ is always bounded within $[1-\epsilon, 1+\epsilon]$, regardless of the sign of the advantage. Theoretically, this ensures the gradient is strictly bounded, avoiding the influence of extreme outliers.

\paragraph{Self-normalization.}
While unconditional clipping mitigates individual outliers, the gradient can still oscillate violently between the lower and upper bounds due to group-level variance. As we show in Sec.~\ref{sec:theory}, with unconditional clipping alone, the gradient frequently approaches the preset upper bound. This creates a trade-off where a loose bound leads to instability, while a tight bound conceals the true importance signal and slows learning.

To address group-level anomalies, we replace the fixed group size normalizer $G$ with the sum of clipped ratios, $\sum_{i=1}^G \text{clip}_\epsilon(\hat{\rho}_i)$. By rescaling, we confine the update to the convex hull of the per-sample gradients, rather than allowing the magnitude to oscillate between preset bounds. The gradient update for StableDRL is formulated as:
\begin{equation}
\label{eq:stabledrl_objective}
\nabla_\theta\mathcal{J}_{\text{Ours}} = \mathbb{E}\left[ \frac{1}{\sum_{i=1}^G \text{clip}_\epsilon(\hat{\rho}_{i})} \sum_{j=1}^{G} \text{clip}_\epsilon(\hat{\rho}_{j}) A_{j} g_j \right].
\end{equation}

StableDRL is simple yet effective to suppresses gradient spikes, preventing training from entering the instability loop.

\subsection{Theoretical Analysis}
\label{sec:theory}

We explain why GRPO becomes unstable when importance ratios are computed from noisy likelihood proxies in dLLMs.
Our analysis models a self-reinforcing loop between \emph{estimation noise}, \emph{gradient spikes}, and \emph{policy drift}.
We first show that, under GRPO's asymmetric unclipping on negative-advantage samples, the update norm has a nonzero probability
of exceeding any threshold $H$. We then show a feedback mechanism: once a spike-induced step increases a drift state, the derived lower bound on the spike probability is nondecreasing for later inner steps on the same rollout group.

\paragraph{Notations.}
Fix a rollout group $\mathcal{B}=\{x_1,\ldots,x_G\}$ sampled from $\pi_{\theta_{\mathrm{old}}}$, and consider inner updates $\theta_0=\theta_{\mathrm{old}},\theta_1,\theta_2,\ldots$ on this \emph{fixed} group.
Following Sec.~3.1, let $\hat\rho_{i,j}=\exp(\Delta\mathcal{L}_{i,j}+\Delta\eta_{i,j})$ be the estimated importance ratio at inner step $i$, where $\Delta\mathcal{L}_{i,j}$ is the noise-free drift and $\Delta\eta_{i,j}$ is the log-ratio estimation noise.
Let $g_{i,j}$ denote the advantage-weighted proxy gradient used in the update.
We use a single constant $B$ such that $\|g_{i,j}\|\le B$ for all inner steps $i$ and samples $j$, which is standard in practice.
Define the negative-advantage set $\mathcal{N}=\{j: A_j\le -a_0\}$ for some $a_0>0$, and the drift state $D_i := \max_{j\in\mathcal{N}} \Delta\mathcal{L}_{i,j}$.

\paragraph{Uniform tail envelope.} To address the non-stationarity of noise across steps, we assume the right tails of $\Delta\eta_{i,j}$ admit a common lower bound (App.~\ref{app:grpo_instability_loop_proof}).
This assumption allows us to lower-bound the spike probability using a time-invariant function of the drift state $D_i$, ensuring that the instability risk is strictly defined by the magnitude of the drift.

\paragraph{Why GRPO can spike.}
Standard GRPO allows \emph{unclipped} multipliers on negative-advantage samples when $A_j<0$ and $\hat\rho_{i,j}>1+\epsilon$.
We first establish that gradient spikes are statistically inevitable and tightly coupled to the drift state.

\begin{lemma}[\textit{Informal}, \textbf{existence of drift-dependent spike probability}]
\label{lem:grpo_spike_prob}
In inner step $i$ of GRPO, for any threshold $H>0$, there exists a lower bound $P_i(H)\in(0,1)$ such that
\begin{equation}
\label{eq:spike_prob_lower_main}
\Pr\!\Big(\big\|\nabla_\theta\mathcal{J}_{\mathrm{GRPO}}\big\|\ge H \,\Big|\, D_i\Big)\ \ge\ P_i(H).
\end{equation}
Moreover, under the common tail-envelope condition on $\Delta\eta_{i,j}$, the bound $P_i(H)$ can be chosen as the
\emph{same nondecreasing function of $D_i$ for all inner steps} (see App.~\ref{app:grpo_instability_loop_proof}).
\end{lemma}

As the policy drifts ($D_i\uparrow$), the noise margin required to push an importance ratio above any fixed level shrinks.
Consequently, large-multiplier outliers become increasingly probable.

\begin{theorem}[\textit{Informal}, \textbf{self-reinforcing instability loop}]
\label{thm:grpo_instability_loop}
Consider an inner step $i$ where a gradient spike occurs ($\|\nabla_\theta\mathcal{J}_{\mathrm{GRPO}}\|\ge H$) and the spike is driven by
a single negative-advantage outlier that dominates the group update (sufficient conditions are in
App.~\ref{app:grpo_instability_loop_proof}).
Then the resulting update increases the drift state:
\begin{equation}
\label{eq:drift_increase_main}
D_{i+1} \ge D_i.
\end{equation}
Consequently, since the bound $P_i(H)$ in Lemma~\ref{lem:grpo_spike_prob} is defined via a common tail envelope and is nondecreasing in $D_i$,
the next-step spike lower bound cannot decrease on that realized step:
\begin{equation}
\label{eq:Pi_monotone_main}
P_{i+1}(H) \ge P_i(H).
\end{equation}
\end{theorem}

\paragraph{Clipping alone can saturate.}
Unconditional two-sided clipping prevents unbounded spikes, but may enter a high-frequency ``boundary-hitting'' regime.

\begin{lemma}[\textit{Informal}, \textbf{existence of hitting probability}]
\label{lem:uclip_saturation}
In unconditionally clipped GRPO, the update norm is deterministically bounded by
\begin{equation}
\label{eq:Hmax_main}
H_{\max}=(1+\epsilon)B.
\end{equation}
For any threshold $H$ close to $H_{\max}$, there exists a saturation probability $Q_i(H)$ that the update hits the upper boundary.
Under the same common tail-envelope condition, $Q_i(H)$ can be chosen to be nondecreasing in $D_i$
(see App.~\ref{app:uclip_instability_loop_proof}).
\end{lemma}

\begin{theorem}[\textit{Informal}, \textbf{self-reinforcing hitting loop}]
\label{thm:uclip_loop}
If at inner step $i$ the update saturates near the upper boundary ($\|\nabla_\theta\mathcal{J}_{\mathrm{UC\text{-}GRPO}}\|\ge H$)
and the saturated step induces non-decreasing drift ($D_{i+1}\ge D_i$ under the appendix conditions),
then, by monotonicity of $Q_i(H)$ in Lemma~\ref{lem:uclip_saturation}$,$
\begin{equation}
\label{eq:Qi_monotone_main}
Q_{i+1}(H) \ge Q_i(H).
\end{equation}
\end{theorem}

Thus, clipping alone can trade rare, unbounded spikes for frequent boundary-saturated updates.
This creates a trade-off: a loose upper bound still destabilizes optimization, while an overly tight bound can obscure
the importance-weight signal.

\paragraph{Why StableDRL breaks the loop.}
Finally, we show how StableDRL structurally removes the remaining group-scale randomness.

\begin{theorem}[\textbf{StableDRL}]
\label{thm:selfnorm_convex_hull}
Let $w_{i,j}:=\mathrm{clip}_\epsilon(\hat\rho_{i,j})$ be the clipped weights.
The StableDRL update $\nabla_\theta\mathcal{J}_{\mathrm{Ours}}$ is normalized by the sum of weights.
Since $w_{i,j}>0$, the update always lies in the \textbf{convex hull} of the per-sample directions
$\{g_{i,j}\}$:
\begin{equation}
\label{eq:sn_convex_hull_main}
\big\|\nabla_\theta\mathcal{J}_{\mathrm{Ours}}\big\|
=
\left\|
\frac{\sum_{j=1}^G w_{i,j}\, g_{i,j}}{\sum_{j=1}^G w_{i,j}}
\right\|
\le
\max_j \|g_{i,j}\|
\le B.
\end{equation}
\end{theorem}

Unlike clipping alone, self-normalization explicitly divides out the random group-scale factor
$\frac{1}{G}\sum_j w_{i,j}$, decoupling the update magnitude from group-level weight fluctuations.
This breaks the instability mechanisms formalized in App.~\ref{app:grpo_instability_loop_proof} and
App.~\ref{app:uclip_instability_loop_proof}. In Sec.~\ref{sec:exp_instable}, we empirically validate these explanations.

\subsection{Generalization to Block Diffusion}
\label{subsec:staircase_attention}


Adapting block diffusion~\citep{wu2025fast,cheng2025sdar} to RL creates a dilemma between training efficiency and information leakage. Valid likelihood proxy estimation requires conditioning each block strictly on its clean history. Naive iterative implementations are prohibitively slow ($\mathcal{O}(K)$), while standard parallel attention invalidates gradient signals by allowing tokens to "cheat" and attend to their own ground truth.


To resolve this, we introduce \textit{staircase attention}, a structured masking primitive that enables leakage-free, single-pass evaluation. By utilizing a dual-stream input of frozen clean context and corrupted target, the mask enforces strict conditional independence through a unique geometry (Figure~\ref{fig:staircase_mask}). A block-lower-triangular "staircase" grants target tokens in block $k$ access to the clean history of preceding blocks ($1 \dots k-1$) while mechanically occluding the current block's ground truth. Simultaneously, a block-diagonal component permits parallel, independent denoising within the target stream. This structure satisfies ELBO requirements within a single computational graph ($\mathcal{O}(1)$), rendering full-parameter RL feasible for long-horizon tasks (formal derivation in Appendix~\ref{app:staircase_details}).

\begin{figure}[H]
    \vspace{-0.2cm}
    \centering
    \includegraphics[width=0.50\linewidth]{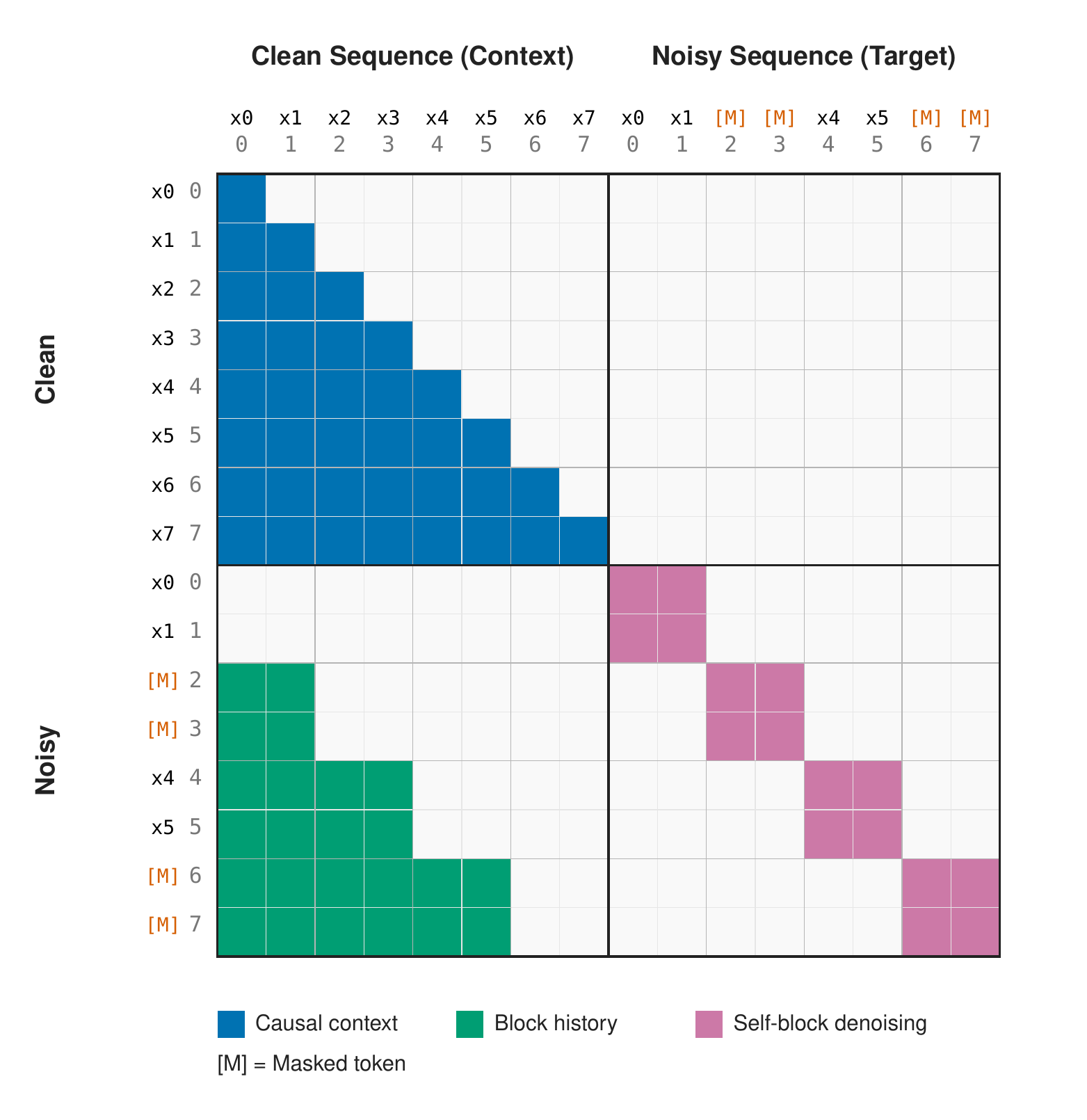}
    \vspace{-20pt}
    \caption{\textbf{Staircase Attention for Efficient Proxy Estimation.} To evaluate the ELBO for block diffusion in a single pass ($O(1)$), we use a dual-stream construction. The \textcolor{blue}{Clean Context} (top rows) provides immutable history. The \textcolor{green}{Corrupted Target} stream (bottom rows) uses a ``staircase'' mask ($M_{\textsc{stair}}$, bottom-left) to attend to valid history without peeking at the ground truth of the current block. The target self-attention ($M_{\textsc{intra}}$, bottom-right) is block-diagonal, ensuring independent parallel denoising.}
    \label{fig:staircase_mask}
    \vspace{-15pt}
\end{figure}

\subsection{Pratical Implementations}\label{sec:implementation}

\paragraph{Score-function surrogates.}
Since dLLMs lack a tractable $\nabla_\theta \log \pi_\theta$, StableDRL reweights gradient directions provided by stable score surrogates, current state-of-the-art dLLM RL methods SPG~\citep{yang2025mmada, ou2025espo,wang2025spg}. For full-attention dLLMs, we implement \emph{block-wise masking}~\citep{wang2025spg} by sampling structured mask blocks consistent with the inference denoising schedule. 
For block diffusion, we sample random positions to mask since the staircase attention runs in $\mathcal{O}(1)$.


\paragraph{Numerically stable log-space weights.}
Direct computation of Eq.~\eqref{eq:stabledrl_objective} is numerically unstable due to the exponentiation of noisy log-ratios. We strictly compute weights in log-space. We define the clipped log-ratios $\tilde{\ell}_{j} = \text{clip}(\widehat{\mathcal{L}}_{\theta}(x_{j}) - \widehat{\mathcal{L}}_{\theta_{\mathrm{old}}}(x_{j}), \log(1-\epsilon), \log(1+\epsilon))$ and compute the normalized coefficients via a stable softmax $\exp\!\big( \tilde{\ell}_{j} - \mathrm{LSE}(\{\tilde{\ell}_{k}\}_{k=1}^G) \big)$, where $\mathrm{LSE}(\cdot)$ is the Log-Sum-Exp function. This \emph{clip-then-softmax} approach preserves numerical precision even when raw probability ratios would underflow or overflow, ensuring stable optimization in mixed-precision training.

\vspace{-0.1cm}
\section{Experiments}\label{sec:exp}

We evaluate StableDRL on two diffusion diffusion architectures: full-attention masked diffusion (\texttt{LLaDA-8B-Instruct}) and semi-autoregressive block diffusion (\texttt{SDAR-8B}). Specifically, we (i) empirically verify the theoretical instability mechanisms analyzed in Sec.~\ref{sec:theory}, (ii) show state-of-the-art reasoning performance on standard benchmarks, and (iii) confirm the architectural generality of our framework. Ablation studies are also 
conduct to dissect the effects of unconditional clipping and self-normalization to training stability.

\paragraph{Experimental setup.} For both \texttt{LLaDA-8B-Instruct} and \texttt{SDAR-8B}, We perform RL fine-tuning using the AdamW optimizer with a learning rate of $1.0 \times 10^{-6}$. To ensure optimization stability while maintaining sample efficiency, we set the unconditional importance weight clipping threshold to $\epsilon=5$ by default. For a comprehensive description of the training infrastructure, model configurations, and hyperparameters, please refer to Appendix~\ref{app:training_setup}.


\subsection{Empirical Verification of Instability Mechanisms}
\label{sec:exp_instable}

\begin{figure*}[h]
    \centering
    \includegraphics[width=0.95\linewidth]{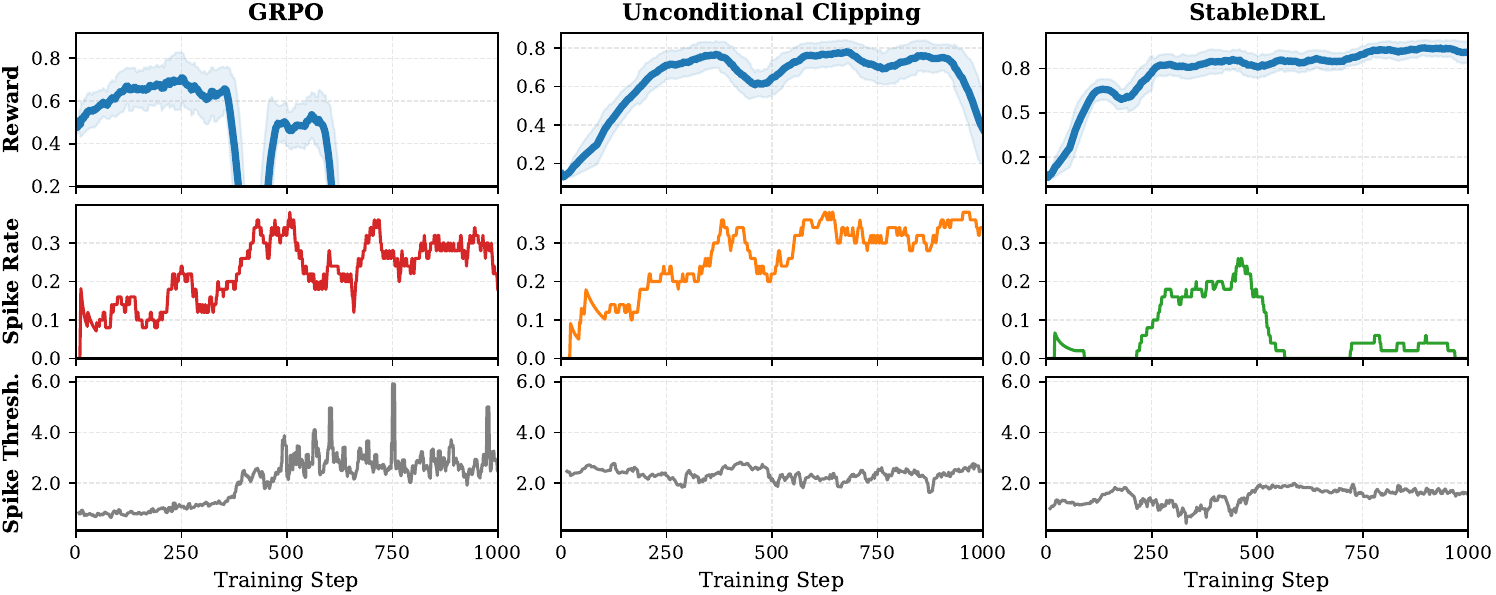}
    \caption{\textbf{Verification of Instability Mechanisms across Methods.} \textbf{Left Column (GRPO):} Unbounded drift (bottom) fuels an accelerating spike rate (middle), causing reward collapse (top). \textbf{Middle Column (Unconditional Clipping):} Clipping saturates the drift (bottom) but induces a high-frequency, stochastic spike regime (middle) that destabilizes learning (top). \textbf{Right Column (StableDRL):} Our method maintains a low, stable spike rate (middle) decoupled from drift (bottom), resulting in monotonic reward improvement (top).}
    \label{fig:instability_analysis}
\end{figure*}

\paragraph{Experimental setup.}
To bridge the gap between our theoretical analysis of self-reinforcing instability in Sec~\ref{sec:theory} and observed training dynamics, we perform a fine-grained analysis of gradient norm evolution. We introduce the \textit{Relative Gradient Spike Rate} to quantify instability. A step $t$ is classified as a ``spike'' if the gradient norm exceeds the local moving average by a margin $\delta$:
\begin{equation*}
    \mathbb{I}_{\text{spike}}^{(t)} = \mathds{1}\left[ \|g_t\| > (1+\delta) \cdot \frac{1}{W} \sum_{k=1}^W \|g_{t-k}\| \right],
    \label{eq:spike_rate}
\end{equation*}
where we set window $W=50$ and $\delta=0.3$. We compare the evolution of reward, spikes rate and the next spike threshold, $(1+\delta) \cdot \frac{1}{W} \sum_{k=1}^W \|g_{t-k}\|$, for GRPO, unconditional clipping, and StableDRL in Figure~\ref{fig:instability_analysis}. We take ESPO~\cite{ou2025espo} as a representative for GRPO and follow its original training setting. We implement unconditional clipping within our training framework with the same clipping threshold $\epsilon$ and settings. For all three methods\footnote{For this Experiment, we train GRPO on GSM8K. For Unconditional Clipping and StableDRL, we train on CountDown.}, we perform full RL finetuning. 

As illustrated in Figure~\ref{fig:instability_analysis}, GRPO exhibits divergent behavior in which unbounded importance ratios drive a steadily rise spike threshold, pushing the optimization into a high-variance regime that requires increasingly large gradients to destabilizing the reward signal. While unconditional clipping bounds the threshold, it results in a high saturation rate where gradient norms frequently impact the clipping limit. This intermittent saturation introduces oscillatory dynamics that corrupt AdamW's momentum history, resulting in reward collapse. Conversely, StableDRL employs a structural convex-hull constraint that maintain a low and stable spike threshold, suppressing relative spikes and enables smooth, monotonic reward gain.

\subsection{Mian Results}
\subsubsection{Full-Attention Diffusion Results}
\label{sec:exp_full_attn}

\begin{table*}[h]
\centering
\caption{\textbf{State-of-the-Art Reasoning Performance on LLaDA-8B-Instruct.} We report pass@1 accuracy under three decoding budgets ($N \in \{128, 256, 512\}$) and the mean performance (\textbf{Avg}) for each dataset. \textbf{Bold} denote the best result and \underline{underline} the second best. StableDRL achieves the highest average accuracy on all four benchmarks, demonstrating superior consistency across generation lengths.}
\label{tab:fullattn_results}
\small
\renewcommand{\arraystretch}{1.15}
\setlength{\tabcolsep}{3.0pt}

\resizebox{\textwidth}{!}{%
\begin{tabular}{l cccc cccc cccc cccc}
\toprule
& \multicolumn{4}{c}{\textbf{\footnotesize{GSM8K}}} & \multicolumn{4}{c}{\textbf{\footnotesize{MATH500}}} & \multicolumn{4}{c}{\textbf{\footnotesize{Countdown}}} & \multicolumn{4}{c}{\textbf{\footnotesize{Sudoku}}} \\
\cmidrule(lr){2-5}\cmidrule(lr){6-9}\cmidrule(lr){10-13}\cmidrule(lr){14-17}
\textbf{Model / Seq Len} & 128 & 256 & 512 & \textbf{Avg} & 128 & 256 & 512 & \textbf{Avg} & 128 & 256 & 512 & \textbf{Avg} & 128 & 256 & 512 & \textbf{Avg} \\
\midrule
LLaDA-8B-Inst. 
& 69.5 & 77.2 & 79.8 & 75.5 
& 28.2 & 32.4 & 34.6 & 31.7 
& 18.8 & 16.8 & 16.8 & 17.5 
& 5.7 & 27.7 & 26.2 & 19.9 \\
LLaDA-1.5 
& 70.4 & 80.5 & 81.9 & 77.6 
& 26.8 & 32.2 & 35.8 & 31.6 
& 31.9 & 21.1 & 21.5 & 24.8 
& 7.4 & 26.9 & 29.0 & 21.1 \\
D1 
& 72.2 & 80.6 & 81.3 & 78.0 
& 31.4 & 36.0 & 39.4 & 35.6 
& 30.9 & 30.9 & 34.4 & 32.1 
& 7.2 & 32.5 & 29.3 & 23.0 \\
WD1 
& 74.6 & 81.5 & 83.0 & 79.7 
& 31.0 & 37.4 & 39.0 & 35.8 
& 48.8 & 52.3 & 50.8 & 50.6 
& 33.1 & 32.1 & 22.5 & 29.2 \\
UniGRPO 
& 74.9 & 82.5 & 82.7 & 80.0 
& 32.4 & 37.4 & 39.4 & 36.4 
& 44.5 & 43.0 & 57.0 & 48.2 
& 59.0 & 67.0 & 62.9 & 63.0 \\
ESPO 
& \underline{80.0} & 82.3 & 83.7 & 82.0 
& \underline{36.0} & 39.0 & \underline{43.4} & \underline{39.5} 
& \textbf{81.6} & \underline{82.0} & \underline{79.3} & \underline{81.0} 
& \textbf{92.7} & 84.7 & 80.5 & 86.0 \\
SPG w/ Mixture 
& 78.5 & \underline{86.1} & \underline{84.5} & \underline{83.0} 
& 33.4 & \underline{40.0} & 41.8 & 38.4 
& 68.8 & 70.7 & 70.3 & 69.9 
& 82.9 & \textbf{94.0} & \textbf{93.1} & \underline{90.0} \\
\midrule
\textbf{StableDRL (Ours)} 
& \textbf{80.2} & \textbf{86.2} & \textbf{86.3} & \textbf{84.2} 
& \textbf{36.2} & \textbf{45.2} & \textbf{44.0} & \textbf{41.8} 
& \underline{81.3} & \textbf{84.4} & \textbf{84.8} & \textbf{83.5} 
& \underline{91.9} & \underline{92.4} & \underline{90.1} & \textbf{91.5} \\
\bottomrule
\end{tabular}%
}
\end{table*}

\paragraph{Experimental setup.}
We follow the experimental protocol of ESPO~\citep{ou2025espo} and SPG~\citep{wang2025spg}, which build on the D1 and WD1 setup ~\cite{zhao2025d1, tang2025wd1}: benchmarks include GSM8K~\cite{cobbe2021training}, MATH500~\cite{lightman2023lets}, Countdown~\cite{pan2025tinyzero}, and Sudoku~\cite{arel2025sudoku}, with the same train and test splits, and evaluation procedure. Concretely, we evaluate at generation lengths $\{128,256,512\}$ and use confidence-based semi-autoregressive decoding with block size $32$ for both RL rollouts and evaluation. We set $\epsilon=5$ for the best performance.

\paragraph{Baselines.} We compare StableDRL against a representative suite of reinforcement learning algorithms for dLLMs. Baselines include \textbf{D1}~\citep{zhao2025d1} and \textbf{UniGRPO}~\citep{yang2025mmada}, which adapt theGRPO framework by approximating the intractable log-likelihood via one-step unmasking or MC estimation of the ELBO. We also include \textbf{WD1}~\citep{tang2025wd1}, which formulates a weighted policy optimization objective to avoid direct likelihood estimation. Finally, we benchmark against \textbf{SPG}~\citep{wang2025spg}, which mitigates gradient bias by sandwiching the policy objective between a tractable Evidence Upper Bound (EUBO) for negative rewards and the ELBO for positive rewards. All methods are initialized from the \texttt{LLaDA-8B-Instruct}~\citep{nie2025large}.

\paragraph{Enabling stable full fine-tuning.}
Unlike ESPO or SPG, which mitigate instability via LoRA or early stopping, we \emph{fully fine-tune} \texttt{LLaDA-8B-Instruct} by explicitly suppressing the gradient spikes that typically destabilize training. This allows StableDRL to optimize the entire model backbone, better unlocking the latent reasoning capabilities of the dLLM. Notably, while our RL training is conducted at a sequence length of 256 tokens, the resulting model achieves consistent, high performance across all evaluated generation lengths ($128$ to $512$ tokens). This suggests that stable, full-parameter reinforcement learning fosters superior length generalization compared to parameter-efficient or variance-constrained alternatives.



\paragraph{State-of-the-Art performances.}
Table~\ref{tab:fullattn_results} demonstrates that StableDRL establishes a new state-of-the-art by achieving the highest average accuracy across all decoding budgets. Specifically, in complex reasoning (MATH500), it secures an average accuracy of 41.8\%, outperforming ESPO and SPG, with a notable +5.2\% margin over SPG at the 256-token budget. In long-horizon planning (Countdown), StableDRL overcomes the off-policy drift that plagues SPG, delivering a massive +13.7\% gain at 256 tokens to reach 84.4\%. Furthermore, unlike baselines such as ESPO that fluctuate significantly on consistency tasks, StableDRL maintains robustness across all lengths, achieving top average scores on both GSM8K (84.2\%) and Sudoku ({91.5\%). These results confirm that resolving ``Noise-Drift" instability is critical for scaling RL in dLLMs.


\subsubsection{Generalization to Block Diffusion}
\label{sec:exp_block_diffusion}

\begin{table}[h]
    \centering
    \caption{\textbf{Block Diffusion Reasoning Performance.} Pass@1 accuracy on MATH500, GSM8K, and AIME, comparing StableDRL (SDAR-8B backbone) against AR baselines (Qwen3) and prior Block Diffusion methods. StableDRL notably outperforms the strong Qwen3-8B AR model on the rigorous AIME benchmark.}
    \label{tab:block_diffusion_results}
    \small
    \setlength{\tabcolsep}{3.5pt}
    \begin{tabular}{lccc}
    \toprule
    \textbf{Method} & \textbf{MATH500} & \textbf{GSM8K} & \textbf{AIME 24} \\
    \midrule
    \multicolumn{4}{l}{\textit{\textbf{Autoregressive (AR) Baselines}}} \\
    Qwen3-4B & 74.1 & 90.7 & 12.9 \\
    Qwen3-8B & 78.4 & \textbf{92.8} & 10.0 \\
    \midrule
    \multicolumn{4}{l}{\textit{\textbf{Block Diffusion (Dynamic Sampling)}}} \\
    SDAR-8B (Base) & 70.6 & 90.4 & \phantom{0}8.3 \\
    Trado & 75.0 & 91.2 & 11.0 \\
    \textbf{StableDRL (Ours)} & 77.8 & 92.1 & 13.3 \\
    \midrule
    \multicolumn{4}{l}{\textit{\textbf{Block Diffusion (Static Sampling)}}} \\
    SDAR-8B (Base) & 75.4 & 91.1 & 11.8 \\
    Trado & 78.5 & 92.3 & 13.3 \\
    \textbf{StableDRL (Ours)} & \textbf{79.2} & 92.4 & \textbf{16.7} \\
    \bottomrule
    \end{tabular}
    \vspace{-0.3cm}
\end{table}

To demonstrate the architectural generality of our framework, we instantiate StableDRL on the \emph{SDAR-8B-Chat} block diffusion model~\citep{cheng2025sdar}. We utilize Staircase Attention (Sec.~\ref{subsec:staircase_attention}) to enable scalable, leakage-free proxy estimation during training.

\paragraph{Experimental setup.}
We follow TraceRL’s training and evaluation conventions for the \emph{SDAR-8B-Chat}~\cite{cheng2025sdar} model ($B=4$). Each RL sampling iteration generates $16$ trajectories per prompt using \emph{dynamic sampling} ($T=0.9$, temp $1.0$). We train on the selected MATH training data~\cite{lightman2023lets}. Evaluation uses both (i) \emph{static} (greedy block-wise) and (ii) \emph{dynamic} (temperature $1.0$) sampling.

\paragraph{Baselines.} 
We benchmark against the supervised base model \textbf{SDAR-8B}, \textbf{Trado}~\citep{wang2025revolutionizing} For fairness, we adopt the model with just TraceRL training to compare, and the autoregressive \textbf{Qwen3}~\cite{yang2025qwen3technicalreport} (4B and 8B Base) to contextualize performance against standard LLMs. We exclude \textbf{DiRL}~\citep{zhu2026dirlefficientposttrainingframework} from our comparison, as it utilizes a fundamentally different data regime and a complex two-stage training pipeline.

\paragraph{Performance analysis.}
Table~\ref{tab:block_diffusion_results} reports the comparative results. StableDRL consistently outperforms prior block diffusion methods. Notably, on the rigorous \emph{AIME 2024} benchmark, StableDRL achieves \emph{16.7\%} (Static), significantly surpassing the base model (11.8\%), Trado (13.3\%), and even the \emph{autoregressive Qwen3-8B} (10.0\%). This indicates that stable on-policy RL can unlock reasoning capabilities often dormant in supervised baselines. Furthermore, while Trado degrades significantly under dynamic sampling (dropping to 11.0\% on AIME), StableDRL maintains superior robustness (\emph{13.3\%}), indicating it effectively shapes the full probability landscape rather than merely optimizing the mode.\footnote{To enable computationally feasible training, we employ custom \texttt{JetEngine}~\cite{cheng2025sdar} inference kernels.}

\subsection{Stress testing exploding importance ratios.}

\begin{figure*}[h]
    \centering
    \includegraphics[width=\textwidth]{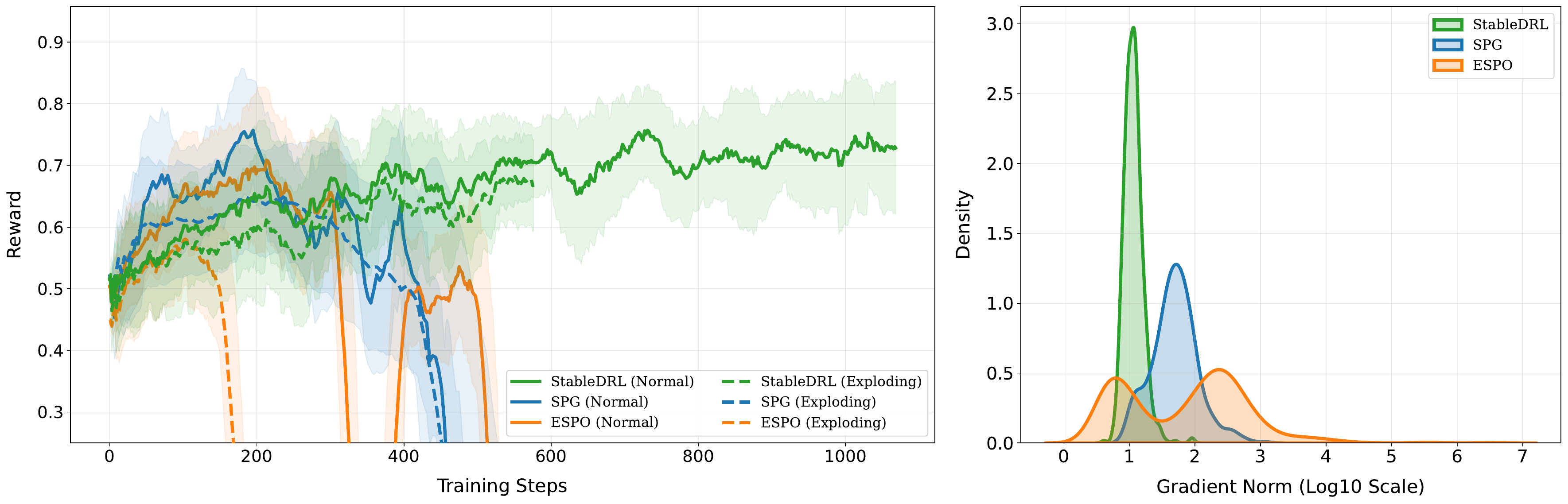}
    \caption{\textbf{Robustness to Proxy Noise: The ``Exploding Weight'' Stress Test (GSM8K).}
    We compare training stability under standard conditions (``Normal'', solid lines) versus an adversarial regime where importance weight variance is artificially amplified (``Exploding'', dashed lines; see App.~\ref{app:stress_test}).
    \textbf{(Left) Reward Trajectories:} StableDRL (\textcolor{green}{Green}) demonstrates \emph{invariant stability}, maintaining monotonic improvement in both regimes. In contrast, ESPO (\textcolor{orange}{Orange}) suffers immediate, \textbf{noise-accelerated collapse}, confirming its sensitivity to ratio outliers. SPG (\textcolor{blue}{Blue}) degrades in both settings, indicating that avoiding ratios (to reduce variance) fatally exposes the model to off-policy bias.
    \textbf{(Right) Gradient Norm Density:} Visualizing the failure mechanism. StableDRL maintains a condensed, low-variance gradient distribution. Conversely, ESPO exhibits a heavy right tail of explosive updates (log-norm $>3$), confirming that the ``Asymmetric Clipping Failure'' allows noise spikes to propagate unchecked.}
    \label{fig:iw_stress_explode}
\end{figure*}

A central hypothesis of this work is that dLLM training instability is driven by heavy-tailed importance weights ($\hat{\rho}$) derived from stochastic ELBO proxies. To isolate this factor, we design an adversarial \emph{Exploding Weight Stress Test} (protocol details in Appendix~\ref{app:stress_test}). This protocol synthesizes ``Exploding'' weights for a subset of trajectories by pairing ``easy'' masking patterns (high ELBO) with ``hard'' masking patterns (low ELBO), amplifying proxy variance without altering the ground-truth data or rewards. Figure~\ref{fig:iw_stress_explode} compares StableDRL, SPG, and ESPO under both ``Normal'' (unbiased) and ``Exploding'' conditions.

\paragraph{StableDRL (Ours): Invariant Stability.}
StableDRL is robust: in the Normal setting (green solid), it achieves the highest final reward; under Exploding weights (green dashed), training remains stable and monotonic, with only a minor performance degradation.

\paragraph{ESPO: Noise-Accelerated Collapse.}
ESPO is highly sensitive to proxy noise: in the Normal setting (orange solid), it collapses later in training, while under Exploding weights (orange dashed) collapse is immediate and catastrophic. This supports our diagnosis that GRPO-style conditional clipping is a primary failure mode under heavy-tailed proxy noise (Sec.~\ref{sec:loop_analysis}).

\paragraph{SPG: Bias-Induced Failure.}
SPG (blue) collapses in \emph{both} settings. Since SPG reuses rollouts without importance-sampling correction (implicitly assuming $\rho=1$), it avoids weight explosions but accumulates off-policy bias as the policy drifts, leading to degradation regardless of proxy noise level.

\subsection{Ablation Studies}
\label{sec:ablation}

\paragraph{Dissecting the stability mechanisms.}
To verify the contributions of Unconditional Clipping and Group Self-Normalization, we analyze the training dynamics on Countdown (Figure~\ref{fig:ablation_mechanisms}). We observe that removing unconditional clipping leads to rapid training failure, as noise-induced outliers dominate the convex combination of gradients. Conversely, removing self-normalization while retaining clipping causes the aggregated update magnitude to oscillate significantly between bounds due to estimation noise, distorting the AdamW momentum history and eventually leading to reward collapse. Only StableDRL, which combines magnitude bounding with geometric constraints, yields a stable and monotonic learning curve.

\begin{figure*}[t]
    \vspace{-15pt}
    \centering
    \begin{minipage}[t]{0.49\textwidth}
        \centering
        \includegraphics[width=\linewidth]{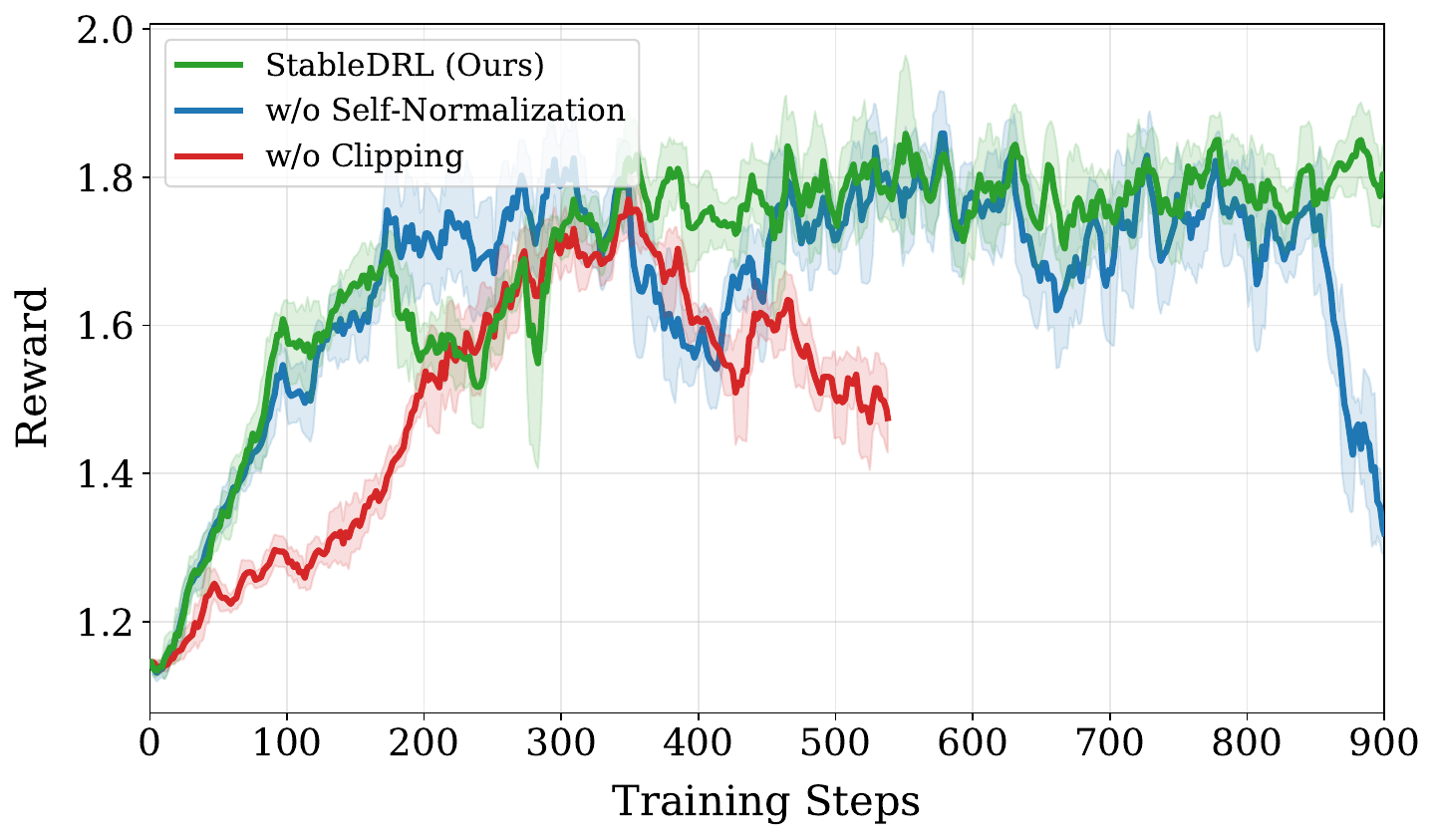}
        \caption{\textbf{Deconstructing Stability Mechanisms.} We isolate the effect of Clipping and Self-Normalization on GSM8K. \textbf{w/o Self-Norm} (Blue): Retaining clipping prevents immediate explosion, but random group scale induces high-variance oscillation that distorts momentum. \textbf{w/o Clipping} (Red): We early-stopped this experiment once observing an unrecoverable collapse in training rewards. Self-normalization alone fails because single noise outliers dominate the convex combination ($\alpha \to 1$), effectively collapsing the sample size and causing rapid failure. 
        \textbf{StableDRL} (Green): Combining both controls yields monotonic stability.}
        \label{fig:ablation_mechanisms}
    \end{minipage}\hfill
    \begin{minipage}[t]{0.49\textwidth}
        \centering
        \includegraphics[width=\linewidth]{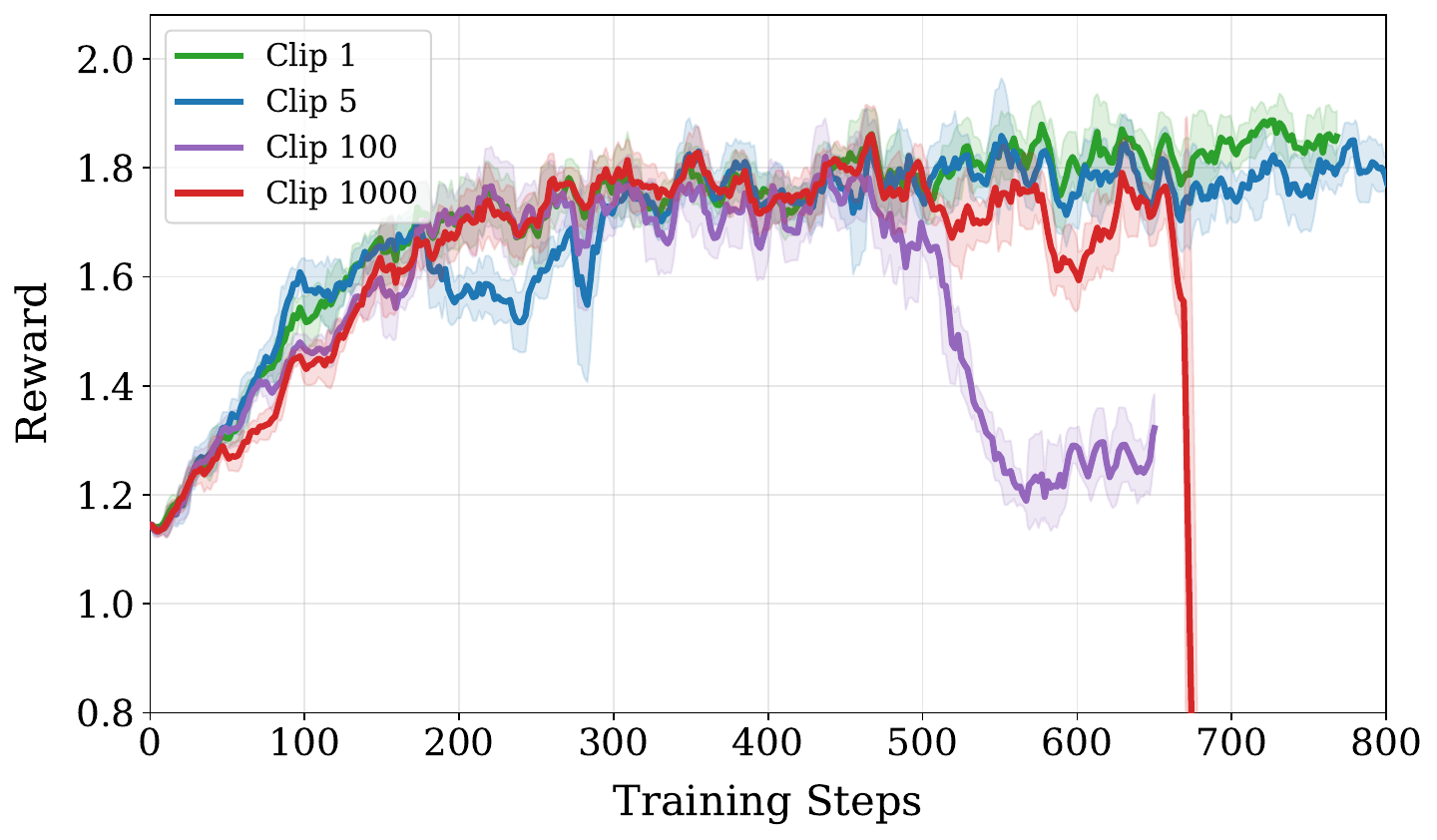}
        \caption{\textbf{Sensitivity to Trust Region Size ($\epsilon$).} We evaluate performance across varying clipping thresholds. 
        \textbf{Small Thresholds} ($\epsilon \in \{1, 5\}$): Training remains stable, with $\epsilon=5$ (Blue) offering a superior bias-variance trade-off compared to the stricter $\epsilon=1$ (Green). 
        \textbf{Large Thresholds} ($\epsilon \in \{100, 1000\}$): As constraints loosen, the ``Trapdoor'' failure re-emerges. Higher thresholds (Purple, Red) allow noise spikes sufficient leverage to destabilize the policy, resulting in sudden, catastrophic collapse.}
        \label{fig:ablation_clipping}
    \end{minipage}
    \vspace{-15pt}
\end{figure*}

\paragraph{Sensitivity to trust region tightness ($\epsilon$).}
We further analyze the trade-off between stability and learning speed across varying clipping thresholds (Figure~\ref{fig:ablation_clipping}). In the small threshold regime ($\epsilon \in \{1, 5\}$), training remains stable, with $\epsilon=5$ providing a superior exploration-stability trade-off than the stricter $\epsilon=1$, achieving faster convergence and higher final rewards by preserving valid learning signals. However, as constraints loosen significantly ($\epsilon \in \{100, 1000\}$),  higher thresholds allow noise spikes sufficient leverage to destabilize the policy before clipping takes effect, resulting in sudden and catastrophic collapse.

\vspace{-0.1cm}
\section{Related Work}
\label{sec:related_work}

\vspace{-0.1cm}
\paragraph{RL post-training for LMs.}
Policy-gradient RL underpins modern alignment and post-training pipelines~\citep{williams1992reinforce,schulman2015trpo,schulman2017ppo}.
RLHF popularized preference-based alignment for AR LMs~\citep{ouyang2022instructgpt}, while RL with verifiable rewards has shown strong gains for mathematical reasoning and long-form solutions~\citep{shao2024deepseekmath,deepseek2025r1}. In the AR context, IR correction~\cite{lingteam2025stepevolvesscalingreinforcement, zheng2025gspo}, targets the staleness in behavior and target policy.


\paragraph{RL for diffusion LMs.}
Other methods such as LLaDa 1.5~\cite{zhu2025llada} proposes a off-policy RL algorithm, which often yields less performance gain then the on-policy ones~\cite{zhao2025d1, ou2025espo}. Meanwhile, models like MDPO~\cite{he2025mdpo} models the diffusion process as a formal markov process, which is computationally expensive and face significant challenges in scaling to large-parameter models.

\paragraph{Importance sampling robustness and off-policy stabilization.}
Variance and tail behavior of importance weights are classical concerns in Monte Carlo and off-policy estimation~\citep{hesterberg1988advances,owen2013monte,elvira2021importance}.
Truncation and clipping control extreme weights~\citep{ionides2008truncated}, while diagnostics and smoothing characterize heavy-tail regimes~\citep{vehtari2015psis}.
In deep RL, clipped corrections such as V-trace and Retrace mitigate off-policy variance and improve stability~\citep{espeholt2018impala,munos2016retrace,liu2018breaking,greensmith2004variance}.
Our work adapts these robustness principles to the \emph{proxy-ratio} setting of dLLM RL, where likelihood estimation noise is exponentiated inside the importance weights.

\section{Conclusion}
\label{sec:conclusion}


This paper studies the instability of Group Relative Policy Optimization (GRPO) when applied to discrete diffusion large language models. We identify that GRPO instability in dLLMs stems from the noisy Monte Carlo importance ratio estimation, which triggers a self-reinforcing instability loop of gradient spikes and policy drift. To break this loop, we propose StableDRL, which employs unconditional clipping and self-normalization to eliminate spikes. Extensive experiments demonstrate that our proposed approach effectively stabilizes the training and significantly unlocks the reasoning potential of dLLMs.

\newpage

\bibliography{ref}
\bibliographystyle{assets/plainnat}

\newpage
\appendix
\onecolumn

\section{Details on Staircase Attention and Proxy Estimation}
\label{app:staircase_details}

In this section, we provide the theoretical details for adapting Reinforcement Learning to Block Diffusion models. We discuss the Monte Carlo estimation of the objective, the efficiency-leakage dilemma, and the formal construction of the Staircase Attention mask.

\subsection{Monte Carlo Estimation of ELBO}
For a fixed context $c$ and sequence $x$, we estimate likelihood proxies by sampling $m$ perturbations. Let $\xi=(t,M_t)$ collect the internal diffusion randomness (time and mask). A generic Monte Carlo (MC) estimator of the ELBO takes the form:
\begin{equation}
\label{eq:mc_elbo_app}
\widehat{\mathcal{L}}_{\text{ELBO}}(x\mid c;\theta)
=
-\frac{1}{m}\sum_{\tau=1}^{m}
\Big[
w(t_\tau)\sum_{i=1}^{n}\mathbf{1}(M_{t_\tau}^{i}=1)\log \pi_{\theta}(x^{i}\mid x_{t_\tau}^{(\tau)},c)
\Big],
\end{equation}
where $x_{t_\tau}^{(\tau)}$ is produced by the forward process using $(t_\tau,M_{t_\tau})$.
In standard full-attention models, the conditional $\log \pi_{\theta}$ is computed under a full bidirectional mask. However, for block diffusion, we must enforce block-wise conditional independence to ensure the estimator remains a valid lower bound.

\subsection{The Efficiency-Leakage Dilemma}
For a sequence divided into $K$ blocks, an exact ELBO estimate requires conditioning each block $B_k$ strictly on its clean history $x_{<B_k}$.
\begin{itemize}
    \item \textbf{Naive Iterative Implementation ($O(K)$):} This necessitates $K$ separate forward passes, masking future tokens sequentially. For long sequences (e.g., 64 blocks), this increases the training cost linearly, rendering iterative RL prohibitively expensive.
    \item \textbf{Standard Single-Pass (Leakage):} Conversely, standard bidirectional attention allows all tokens to attend to the full sequence. If applied naively in a single pass, denoising tokens in block $B_k$ would attend to the ground-truth representations of their own block, mathematically invalidating the variational bound and the gradient signal.
\end{itemize}


\subsection{Dual-Stream Input and Mask Construction}
To achieve $O(1)$ evaluation without leakage, we employ a dual-stream (``2L'') input construction. We concatenate a \emph{clean context stream} $x_{\text{ctx}}$ (frozen history) and a \emph{corrupted target stream} $x_{\text{tgt}}$ (containing mask tokens). Let $\tilde{x}=[x_{\mathrm{ctx}};\,x_{\mathrm{tgt}}]$ be the combined input of length $2n$.

We define a composite attention mask $M \in \{0,1\}^{2n \times 2n}$ that enforces strict block-causal dependency:
\begin{equation}
M \;=\;
\begin{bmatrix}
M_{\textsc{causal}} & \mathbf{0} \\
M_{\textsc{stair}} & M_{\textsc{intra}}
\end{bmatrix}.
\end{equation}

The components are defined as follows:
\begin{enumerate}
    \item \textbf{Top-Left ($M_{\textsc{causal}}$):} Standard causal mask for the clean context stream (Blue regions in Figure~\ref{fig:staircase_mask}).
    \item \textbf{Top-Right ($\mathbf{0}$):} Zero matrix. The clean context cannot attend to the noisy target.
    \item \textbf{Bottom-Right ($M_{\textsc{intra}}$):} A block-diagonal mask where $(M_{\textsc{intra}})_{ij}=1$ iff target positions $i$ and $j$ belong to the same block. This corresponds to the \textbf{Pink} regions in Figure~\ref{fig:staircase_mask} and enables intra-block denoising.
    \item \textbf{Bottom-Left ($M_{\textsc{stair}}$):} The strictly block-lower-triangular component, corresponding to the \textbf{Green} regions in Figure~\ref{fig:staircase_mask}. For a target token in block $k$ and a context token in block $l$:
    \begin{equation}
        (M_{\textsc{stair}})_{k,l} = 
        \begin{cases} 
        1 & \text{if } l < k \quad \text{(Context: Attend to history)} \\
        0 & \text{if } l \ge k \quad \text{(Context: Occlude current/future)}
        \end{cases}
    \end{equation}
\end{enumerate}
This construction allows us to compute gradients for all $K$ blocks simultaneously while mathematically preserving the autoregressive factorization required by the objective.

\section{Proof of Main Results}
\label{app:proof}

\subsection{Formal theorem statements for Sec.~\ref{sec:theory}}
\label{app:theory_formal}

This subsection presents the formal statements of the instability mechanisms identified in Section~\ref{sec:theory}.
Theorems~\ref{thm:grpo_instability_loop_core}--\ref{thm:selfnorm_convex_hull} outline our theoretical framework in three logical steps.
First, \textbf{Theorem~\ref{thm:grpo_instability_loop_core}} formally characterizes the drift–spike feedback loop inherent to standard GRPO.
Second, \textbf{Theorem~\ref{thm:uclip_instability_loop}} states that two-sided unconditional clipping, while mitigating spikes,  may lead to frequent boundary saturation.
Finally, \textbf{Theorem~\ref{thm:selfnorm_convex_hull}} establishes that self-normalization structurally resolves the remaining random group-scale factor.
Detailed proofs are provided in subsequent subsections.

\paragraph{Mathematical setup.}
Fix a behavior policy $\theta_{\mathrm{old}}$ and a rollout group $\mathcal{B}=\{x_1,\ldots,x_G\}$ sampled from $\pi_{\theta_{\mathrm{old}}}$.
GRPO performs updates $\theta_0=\theta_{\mathrm{old}},\theta_1,\theta_2,\ldots$ on this same fixed group.
Write the estimated importance ratio on sample $x_j$ at step $i$ as
\[
\hat\rho_{i,j}=\exp(\Delta\mathcal{L}_{i,j}+\Delta\eta_{i,j}),
\]
where $\Delta\mathcal{L}_{i,j}=\mathcal{L}_{\theta_i}(x_j)-\mathcal{L}_{\theta_{\mathrm{old}}}(x_j)$ is the noise-free drift and
$\Delta\eta_{i,j}$ is the corresponding log-ratio noise term.
Let $\widehat A_j$ be the fixed, group-relative advantage, and define the negative set
\[
\mathcal{N}=\{j:\widehat A_j\le -a_0\}
\quad \text{for some } a_0>0.
\]
Define the drift state, within-negative spread, and drift-maximizer index
\begin{equation}
\label{eq:app_Di_Si_def}
D_i=\max_{j\in\mathcal{N}}\Delta\mathcal{L}_{i,j},
\qquad
S_i=D_i-\min_{j\in\mathcal{N}}\Delta\mathcal{L}_{i,j},
\qquad
j^\dagger \in \arg\max_{j\in\mathcal{N}}\Delta\mathcal{L}_{i,j}.
\end{equation}
Finally, let $\widehat g_{\mathrm{GRPO},i}$ denote the implemented GRPO update direction at step $i$.

\begin{theorem}[GRPO drift--spike feedback loop]
\label{thm:grpo_instability_loop_core}
Assume the standing Conditions (C1)--(C5) in Appendix~\ref{app:grpo_instability_loop_proof}.
Fix any spike threshold $H>0$ and define
\begin{equation}
\label{eq:core_uH_def}
u_H
:=
\max\!\left\{1+\epsilon,\ \frac{GH}{(1-\lambda)a_0b_0},\ u_0\right\},
\end{equation}
where $u_0$ is the deterministic constant defined in Lemma~\ref{lem:dom_from_moment}.
Define the spike-probability lower bound
\begin{equation}
\label{eq:core_Pi_def_new}
P_i(H) \ :=\ \frac{1}{2}\,\bar F\!\big(\log u_H - D_i\big).
\end{equation}
Then, almost surely (conditioning on $\mathcal F_{i-1}$, i.e., the current inner iterate and the fixed rollout group),
\begin{equation}
\label{eq:core_spike_prob_main_new}
\mathbb{P}\!\big(\|\widehat g_{\mathrm{GRPO},i}\|\ge H \,\big|\, \mathcal F_{i-1}\big)\ \ge\ P_i(H),
\qquad \text{a.s.}
\end{equation}
and $P_i(H)$ is a nondecreasing function of the drift state $D_i$.

Moreover, on any realized step where a single negative-advantage outlier dominates the group update and the local smoothness/geometry
conditions in Appendix~\ref{app:grpo_instability_loop_proof} hold for that realized update, there exist step-dependent scalars
$c_{\mathrm{sup},i}>0$ and $c_{\mathrm{amp},i}\in\mathbb{R}$ and indices $j^\star\in\mathcal{N}$ and $j^\diamond\in\mathcal{N}\setminus\{j^\star\}$ such that
\begin{equation}
\label{eq:core_drift_creation_main}
\begin{aligned}
\mathcal{L}_{\theta_i}(x_{j^\star})-\mathcal{L}_{\theta_{i+1}}(x_{j^\star})
&\ge
c_{\mathrm{sup},i}\,\frac{\hat\rho_{i,j^\star}}{G},\\
D_{i+1}
&\ge
D_i + \big(c_{\mathrm{amp},i}\hat\rho_{i,j^\star}-S_i\big).
\end{aligned}
\end{equation}
In particular, if $c_{\mathrm{amp},i}\hat\rho_{i,j^\star}\ge S_i$, then $D_{i+1}\ge D_i$, hence
\begin{equation}
\label{eq:Pi_monotone_formal}
P_{i+1}(H)\ \ge\ P_i(H)
\qquad\text{(on that realized step).}
\end{equation}
\end{theorem}

\begin{theorem}[Boundary saturation under two-sided clipping]
\label{thm:uclip_instability_loop}
Let $g_{i,j}:=\widehat A_j\nabla_\theta\mathcal{L}_{\theta_i}(x_j)$ and assume $\|g_{i,j}\|\le B$ (Condition (C1)).
Define the two-sided clipped weight
$w_{i,j}:=\mathrm{clip}(\hat\rho_{i,j},\,1-\epsilon,\,1+\epsilon)$ and the clipping-only direction
$\widehat g_{\mathrm{clip},i}:=\frac{1}{G}\sum_{j=1}^G w_{i,j}g_{i,j}$.

First, clipping prevents unbounded spikes: deterministically, $\|\widehat g_{\mathrm{clip},i}\|\le (1+\epsilon)B$.

Second, drift still increases the frequency of hitting the \emph{upper} clipping boundary.
Let $j^\dagger\in\arg\max_{j\in\mathcal{N}}\Delta\mathcal{L}_{i,j}$ be a drift-maximizer in the negative set.
Then, almost surely,
\begin{equation}
\label{eq:core_upper_sat}
\mathbb{P}\!\big(\hat\rho_{i,j^\dagger}\ge 1+\epsilon \,\big|\, \mathcal F_{i-1}\big)
\ \ge\
\bar F\!\big(\log(1+\epsilon)-D_i\big),
\end{equation}
and the right-hand side is nondecreasing in $D_i$.

Under the additional dominance event in Lemma~\ref{lem:ub_dom_clipped} and the same local smoothness/geometry
conditions used in Appendix~\ref{app:grpo_instability_loop_proof}, there exists a step-dependent scalar $c_{\mathrm{amp},i}\in\mathbb{R}$ such that
\[
D_{i+1}\ \ge\ D_i + \big(c_{\mathrm{amp},i}(1+\epsilon)-S_i\big).
\]
Thus clipping alone can replace rare extreme spikes with frequent boundary-saturated updates once drift becomes large.
\end{theorem}

\begin{theorem}[Self-normalization removes the random group-scale factor]
\label{thm:selfnorm_convex_hull}
With the same clipped weights $w_{i,j}$ as in Theorem~\ref{thm:uclip_instability_loop}, define the self-normalized direction
\begin{equation}
\label{eq:core_selfnorm_def}
\begin{aligned}
\widehat g_{\mathrm{sn},i}
&:=
\frac{\sum_{j=1}^G w_{i,j}g_{i,j}}{\sum_{j=1}^G w_{i,j}},\\
\widehat g_{\mathrm{clip},i}
&=
\left(\frac{1}{G}\sum_{j=1}^G w_{i,j}\right)\widehat g_{\mathrm{sn},i}.
\end{aligned}
\end{equation}
Since $w_{i,j}>0$ (because $\hat\rho_{i,j}=\exp(\cdot)>0$) and $\sum_j w_{i,j}>0$, the coefficients
$w_{i,j}/\sum_k w_{i,k}$ form a convex combination, so
$\widehat g_{\mathrm{sn},i}$ always lies in the convex hull of $\{g_{i,j}\}_{j=1}^G$.
In particular, if $\|g_{i,j}\|\le B$ then deterministically $\|\widehat g_{\mathrm{sn},i}\|\le B$.
Thus self-normalization explicitly divides out the random group-scale factor $\frac{1}{G}\sum_j w_{i,j}$ that remains
under clipping-only.
\end{theorem}

\subsection{Proof of Theorem~\ref{thm:grpo_instability_loop_core}}
\label{app:grpo_instability_loop_proof}

We present the proof of Theorem~\ref{thm:grpo_instability_loop_core}.
For clarity, we first state the necessary setup and assumptions.

\paragraph{Deterministic proxy gradients and GRPO effective weights.}
Define deterministic proxy gradients
\[
h_{i,j}:=\nabla_\theta \mathcal{L}_{\theta_i}(x_j),
\qquad
g_{i,j}:=\widehat A_j\,h_{i,j}.
\]
We write the implemented GRPO direction in the equivalent ``effective-weight'' form
\begin{equation}
\label{eq:grpo_effective_weight_form}
\widehat g_{\mathrm{GRPO},i}
:=
\frac{1}{G}\sum_{j=1}^G m_{i,j}\,g_{i,j},
\end{equation}
where the (random) effective multiplier $m_{i,j}$ is
\begin{equation}
\label{eq:def_effective_mij}
m_{i,j}
:=
\begin{cases}
\min(\hat\rho_{i,j},\,1+\epsilon), & \widehat A_j\ge 0,\\
\max(\hat\rho_{i,j},\,1-\epsilon), & \widehat A_j<0.
\end{cases}
\end{equation}
This form is exactly equivalent to the usual GRPO ``min--clip'' surrogate:
for $\widehat A_j\ge 0$ the weight is clipped at $1+\epsilon$, while for $\widehat A_j<0$ the weight is not clipped from above.

We do not assume a particular optimizer beyond the update form
\[
\theta_{i+1}=\theta_i+\eta_0\,\widehat g_{\mathrm{GRPO},i}
\]
for some learning rate $\eta_0>0$.

\paragraph{Filtration.}
Let $\mathcal F_i$ denote the $\sigma$-field generated by all algorithmic randomness up to and including step $i$.
Then $\theta_i$ is $\mathcal F_{i-1}$-measurable; hence each drift value $\Delta\mathcal{L}_{i,j}$ is
$\mathcal F_{i-1}$-measurable.

\paragraph{Standing conditions (C1--C5).}
We work under the following conditions. (C1) is standard and typically enforced by gradient clipping;
(C3) holds when the Monte Carlo proxy evaluations use independent randomness across samples;
and (C5) is \emph{empirically checkable} by monitoring $\sum_{j\neq j^\dagger}m_{i,j}$.

\begin{itemize}
\item \textbf{(C1) Bounded per-sample directions.}
There exists $B<\infty$ such that $\|g_{i,j}\|\le B$ for all inner steps $i$ and all samples $j$.

\item \textbf{(C2) Conditional common right-tail envelope for log-ratio noise.}
For each inner step $i$ and sample $j$, define the conditional survival function
\[
\bar F_{j,i}(z)
:=
\mathbb{P}\big(\Delta\eta_{i,j}\ge z \,\big|\, \mathcal F_{i-1}\big),
\qquad z\in\mathbb{R}.
\]
Assume that for every $j\in\mathcal N$ and every $i$, $\bar F_{j,i}$ has unbounded support in the sense that
$\bar F_{j,i}(z)>0$ for all $z\in\mathbb{R}$ almost surely.
Moreover, assume there exists a deterministic nonincreasing function $\bar F:\mathbb{R}\to(0,1]$
(a \emph{uniform} tail lower envelope) such that almost surely, for all $i$ and all $j\in\mathcal N$,
\[
\bar F_{j,i}(z)\ge \bar F(z)
\qquad \forall z\in\mathbb{R}.
\]

\item \textbf{(C3) Conditional independence across samples.}
For each inner step $i$, conditional on $\mathcal F_{i-1}$, the noises $\{\Delta\eta_{i,j}\}_{j=1}^G$ are independent.
Equivalently, conditional on $\mathcal F_{i-1}$, the ratios $\{\hat\rho_{i,j}\}_{j=1}^G$ (and thus the effective weights $\{m_{i,j}\}$)
are independent across $j$.

\item \textbf{(C4) Nontrivial proxy gradient at the drift-maximizer.}
There exists $b_0>0$ such that for all inner steps $i$, the drift-maximizer in the negative set satisfies
$\|h_{i,j^\dagger}\|\ge b_0$.

\item \textbf{(C5) Residual effective-weight moment control.}
There exists a deterministic constant $W<\infty$ such that for all inner steps $i$,
\[
\mathbb{E}\!\left[\sum_{j\neq j^\dagger} m_{i,j}\ \Big|\ \mathcal F_{i-1}\right]\ \le\ W,
\qquad \text{a.s.}
\]
\end{itemize}

\begin{remark}
Condition (C5) upper-bounds the \emph{expected total effective-weight mass} of the samples \emph{other than the drift-maximizer}.
This quantity is directly measurable in experiments as the sum $\sum_{j\neq j^\dagger}m_{i,j}$.
Controlling this expectation guarantees via Markov's inequality that when $\hat\rho_{i,j^\dagger}$ is large, the drift-maximizer dominates the group update with a constant probability.
\end{remark}

\begin{lemma}[Ratio exceedance identity and drift monotonicity]
\label{lem:app_ratio_exceed}
Fix an inner step $i$, an index $j$, and a threshold $u>0$.
Then
\begin{equation}
\label{eq:app_ratio_exceed_identity}
\mathbb{P}\!\big(\hat\rho_{i,j}\ge u \,\big|\, \mathcal F_{i-1}\big)
=
\mathbb{P}\!\big(\Delta\eta_{i,j} \ge \log u-\Delta\mathcal L_{i,j} \,\big|\, \mathcal F_{i-1}\big)
=
\bar F_{j,i}\!\big(\log u-\Delta\mathcal L_{i,j}\big).
\end{equation}
Moreover, conditional on $\mathcal F_{i-1}$, the map
$\Delta\mathcal L \mapsto \bar F_{j,i}(\log u-\Delta\mathcal L)$ is nondecreasing.
\end{lemma}

\begin{proof}
By definition, $\hat\rho_{i,j}=\exp(\Delta\mathcal L_{i,j}+\Delta\eta_{i,j})$.
Since $\exp(\cdot)$ is strictly increasing,
\[
\{\hat\rho_{i,j}\ge u\}
\iff
\{\Delta\mathcal L_{i,j}+\Delta\eta_{i,j}\ge \log u\}
\iff
\{\Delta\eta_{i,j}\ge \log u-\Delta\mathcal L_{i,j}\}.
\]
Taking conditional probabilities given $\mathcal F_{i-1}$ yields \eqref{eq:app_ratio_exceed_identity}.
For monotonicity, conditional on $\mathcal F_{i-1}$ the survival function $\bar F_{j,i}$ is nonincreasing in its argument,
while $\Delta\mathcal L\mapsto \log u-\Delta\mathcal L$ is strictly decreasing; therefore their composition is nondecreasing.
\end{proof}


\begin{lemma}[Dominance from a large drift-maximizer ratio via a moment bound]
\label{lem:dom_from_moment}
Fix an inner step $i$ and let $j^\dagger\in\arg\max_{j\in\mathcal N}\Delta\mathcal L_{i,j}$.
Fix any $\lambda\in[0,1)$ and define the (step-$i$) residual vector
\[
r_i := \frac{1}{G}\sum_{j\neq j^\dagger} m_{i,j} g_{i,j}.
\]
Assume Conditions (C1)--(C5).
Define
\begin{equation}
\label{eq:u0_def}
u_0 := \frac{2BW}{\lambda a_0 b_0}.
\end{equation}
Then for any $u\ge u_0$,
\begin{equation}
\label{eq:dom_prob_half}
\mathbb{P}\!\left(
\|r_i\|\le \lambda\,\frac{1}{G}\,u\,a_0 b_0
\ \Big|\ \mathcal F_{i-1},\ \hat\rho_{i,j^\dagger}\ge u
\right)\ \ge\ \frac{1}{2},
\qquad \text{a.s.}
\end{equation}
Moreover, on the event $\{\hat\rho_{i,j^\dagger}\ge u\}$ with $u\ge 1+\epsilon$, since $j^\dagger\in\mathcal N$ we have
$m_{i,j^\dagger}=\hat\rho_{i,j^\dagger}$ and thus
\begin{equation}
\label{eq:dom_decomp_deterministic}
\widehat g_{\mathrm{GRPO},i}
=
-\frac{1}{G}\,\hat\rho_{i,j^\dagger}\,|\widehat A_{j^\dagger}|\,h_{i,j^\dagger} + r_i.
\end{equation}
\end{lemma}

\begin{proof}
First, by (C1),
\[
\|r_i\|
=
\left\|\frac{1}{G}\sum_{j\neq j^\dagger} m_{i,j}g_{i,j}\right\|
\le
\frac{1}{G}\sum_{j\neq j^\dagger} m_{i,j}\|g_{i,j}\|
\le
\frac{B}{G}\sum_{j\neq j^\dagger} m_{i,j}.
\]
Therefore, the event
\[
\left\{\sum_{j\neq j^\dagger} m_{i,j} \le \frac{\lambda}{B}\,u\,a_0 b_0\right\}
\]
implies $\|r_i\|\le \lambda \frac{1}{G}u a_0 b_0$.

Next, by (C3), conditional on $\mathcal F_{i-1}$ the collection
$\{m_{i,j}\}_{j\neq j^\dagger}$ is independent of $\hat\rho_{i,j^\dagger}$, hence independent of the event
$\{\hat\rho_{i,j^\dagger}\ge u\}$.
Thus, for any threshold $t>0$,
\[
\mathbb{P}\!\left(\sum_{j\neq j^\dagger} m_{i,j} > t\ \Big|\ \mathcal F_{i-1},\ \hat\rho_{i,j^\dagger}\ge u\right)
=
\mathbb{P}\!\left(\sum_{j\neq j^\dagger} m_{i,j} > t\ \Big|\ \mathcal F_{i-1}\right).
\]
Applying Markov's inequality and (C5) yields
\[
\mathbb{P}\!\left(\sum_{j\neq j^\dagger} m_{i,j} > t\ \Big|\ \mathcal F_{i-1}\right)
\le
\frac{\mathbb{E}\!\left[\sum_{j\neq j^\dagger} m_{i,j}\mid \mathcal F_{i-1}\right]}{t}
\le
\frac{W}{t}.
\]
Choose $t=\frac{\lambda}{B}u a_0 b_0$. If $u\ge u_0=\frac{2BW}{\lambda a_0 b_0}$, then $W/t\le 1/2$, hence
\[
\mathbb{P}\!\left(\sum_{j\neq j^\dagger} m_{i,j} \le \frac{\lambda}{B}u a_0 b_0\ \Big|\ \mathcal F_{i-1},\ \hat\rho_{i,j^\dagger}\ge u\right)
\ge \frac{1}{2}.
\]
Combining with $\|r_i\|\le \frac{B}{G}\sum_{j\neq j^\dagger}m_{i,j}$ gives \eqref{eq:dom_prob_half}.

Finally, on $\{\hat\rho_{i,j^\dagger}\ge u\}$ with $u\ge 1+\epsilon$, since $j^\dagger\in\mathcal N$ we have
$m_{i,j^\dagger}=\max(\hat\rho_{i,j^\dagger},1-\epsilon)=\hat\rho_{i,j^\dagger}$.
Also $g_{i,j^\dagger}=\widehat A_{j^\dagger}h_{i,j^\dagger}=-|\widehat A_{j^\dagger}|h_{i,j^\dagger}$.
Substituting into \eqref{eq:grpo_effective_weight_form} yields \eqref{eq:dom_decomp_deterministic}.
\end{proof}


\begin{lemma}[Dominance implies a gradient spike]
\label{lem:app_dom_implies_spike_new}
Fix an inner step $i$ and let $j^\dagger$ be as above.
Assume (C1) and (C4).
On the event $\{\hat\rho_{i,j^\dagger}\ge u\}$ with $u\ge 1+\epsilon$, and on any event where
\begin{equation}
\label{eq:residual_small_event}
\|r_i\|\le \lambda\,\frac{1}{G}\,u\,a_0 b_0,
\end{equation}
we have
\[
\|\widehat g_{\mathrm{GRPO},i}\|
\ge
(1-\lambda)\,\frac{1}{G}\,u\,a_0 b_0.
\]
In particular, if
\begin{equation}
\label{eq:u_implies_H}
u \ \ge\ \frac{GH}{(1-\lambda)a_0 b_0},
\end{equation}
then $\|\widehat g_{\mathrm{GRPO},i}\|\ge H$ holds on the same event.
\end{lemma}

\begin{proof}
On $\{\hat\rho_{i,j^\dagger}\ge u\}$ with $u\ge 1+\epsilon$, Lemma~\ref{lem:dom_from_moment} gives the decomposition
\[
\widehat g_{\mathrm{GRPO},i}
=
-\frac{1}{G}\,\hat\rho_{i,j^\dagger}\,|\widehat A_{j^\dagger}|\,h_{i,j^\dagger} + r_i.
\]
Apply the reverse triangle inequality:
\[
\|\widehat g_{\mathrm{GRPO},i}\|
\ge
\frac{1}{G}\hat\rho_{i,j^\dagger}|\widehat A_{j^\dagger}|\|h_{i,j^\dagger}\| - \|r_i\|.
\]
Since $j^\dagger\in\mathcal N$ implies $|\widehat A_{j^\dagger}|\ge a_0$, and (C4) gives
$\|h_{i,j^\dagger}\|\ge b_0$, and $\hat\rho_{i,j^\dagger}\ge u$, we obtain
\[
\frac{1}{G}\hat\rho_{i,j^\dagger}|\widehat A_{j^\dagger}|\|h_{i,j^\dagger}\|
\ge
\frac{1}{G}u a_0 b_0.
\]
Together with \eqref{eq:residual_small_event} this yields
\[
\|\widehat g_{\mathrm{GRPO},i}\|
\ge
\frac{1}{G}u a_0 b_0 - \lambda\frac{1}{G}u a_0 b_0
=
(1-\lambda)\frac{1}{G}u a_0 b_0.
\]
If \eqref{eq:u_implies_H} holds, then the right-hand side is at least $H$.
\end{proof}


\begin{lemma}[A drift-monotone lower bound on spike probability]
\label{lem:app_spike_prob_lower_new}
Fix a step $i$ and a spike threshold $H>0$.
Let $j^\dagger\in\arg\max_{j\in\mathcal N}\Delta\mathcal L_{i,j}$ so that $\Delta\mathcal L_{i,j^\dagger}=D_i$.
Define $u_H$ as in \eqref{eq:core_uH_def}.
Assume Conditions (C1)--(C5).
Then, almost surely,
\begin{equation}
\label{eq:spike_prob_bound_new}
\mathbb{P}\!\big(\|\widehat g_{\mathrm{GRPO},i}\|\ge H \,\big|\, \mathcal F_{i-1}\big)
\ \ge\
\frac{1}{2}\cdot \mathbb{P}\!\big(\hat\rho_{i,j^\dagger}\ge u_H \,\big|\, \mathcal F_{i-1}\big).
\end{equation}
Moreover, by Lemma~\ref{lem:app_ratio_exceed} and (C2),
\begin{equation}
\label{eq:ratio_tail_bound_new}
\mathbb{P}\!\big(\hat\rho_{i,j^\dagger}\ge u_H \,\big|\, \mathcal F_{i-1}\big)
=
\bar F_{j^\dagger,i}\!\big(\log u_H - D_i\big)
\ge
\bar F\!\big(\log u_H - D_i\big),
\end{equation}
and the right-hand side is nondecreasing in $D_i$.
\end{lemma}

\begin{proof}
Work conditionally on $\mathcal F_{i-1}$.
Since $u_H\ge 1+\epsilon$ by definition, on the event $\{\hat\rho_{i,j^\dagger}\ge u_H\}$ we have $m_{i,j^\dagger}=\hat\rho_{i,j^\dagger}$.
By Lemma~\ref{lem:dom_from_moment} with $u=u_H$, we have
\[
\mathbb{P}\!\left(\|r_i\|\le \lambda \frac{1}{G}u_H a_0 b_0 \ \Big|\ \mathcal F_{i-1},\ \hat\rho_{i,j^\dagger}\ge u_H\right)\ge \frac12.
\]
On the intersection of $\{\hat\rho_{i,j^\dagger}\ge u_H\}$ and $\{\|r_i\|\le \lambda \frac{1}{G}u_H a_0 b_0\}$,
Lemma~\ref{lem:app_dom_implies_spike_new} implies $\|\widehat g_{\mathrm{GRPO},i}\|\ge H$ because $u_H\ge GH/((1-\lambda)a_0 b_0)$.
Therefore,
\begin{align*}
\mathbb{P}\!\big(\|\widehat g_{\mathrm{GRPO},i}\|\ge H \,\big|\, \mathcal F_{i-1}\big)
&\ge
\mathbb{P}\!\big(\hat\rho_{i,j^\dagger}\ge u_H,\ \|r_i\|\le \lambda \tfrac{1}{G}u_H a_0 b_0 \,\big|\, \mathcal F_{i-1}\big)\\
&=
\mathbb{P}\!\big(\hat\rho_{i,j^\dagger}\ge u_H \,\big|\, \mathcal F_{i-1}\big)\cdot
\mathbb{P}\!\left(\|r_i\|\le \lambda \tfrac{1}{G}u_H a_0 b_0 \ \Big|\ \mathcal F_{i-1},\ \hat\rho_{i,j^\dagger}\ge u_H\right)\\
&\ge
\frac12\cdot \mathbb{P}\!\big(\hat\rho_{i,j^\dagger}\ge u_H \,\big|\, \mathcal F_{i-1}\big),
\end{align*}
which proves \eqref{eq:spike_prob_bound_new}. The tail identity and lower bound \eqref{eq:ratio_tail_bound_new} follow from
Lemma~\ref{lem:app_ratio_exceed} with $\Delta\mathcal L_{i,j^\dagger}=D_i$, and (C2).
Monotonicity in $D_i$ follows from Lemma~\ref{lem:app_ratio_exceed}.
\end{proof}


\begin{lemma}[Quadratic remainder for $L$-smooth functions]
\label{lem:app_smooth_quad}
Let $f:\mathbb{R}^d\to\mathbb{R}$ be differentiable and $L$-smooth on the segment
$\{\theta+t(\theta'-\theta):t\in[0,1]\}$.
Then
\[
f(\theta') \le f(\theta)+\langle\nabla f(\theta),\theta'-\theta\rangle+\frac{L}{2}\|\theta'-\theta\|^2,
\qquad
f(\theta') \ge f(\theta)+\langle\nabla f(\theta),\theta'-\theta\rangle-\frac{L}{2}\|\theta'-\theta\|^2.
\]
\end{lemma}

\begin{proof}
Let $d := \theta'-\theta$ and define the univariate function
\[
\varphi(t) := f(\theta + t d), \qquad t\in[0,1].
\]
Since $f$ is differentiable on the segment $\{\theta+td:t\in[0,1]\}$, $\varphi$ is differentiable and
\[
\varphi'(t) = \left\langle \nabla f(\theta + t d),\, d \right\rangle.
\]
By the fundamental theorem of calculus,
\[
f(\theta')-f(\theta) = \varphi(1)-\varphi(0)=\int_0^1 \varphi'(t)\,dt
= \int_0^1 \left\langle \nabla f(\theta + t d),\, d \right\rangle dt.
\]
Add and subtract $\nabla f(\theta)$ inside the inner product:
\[
f(\theta')-f(\theta)
=
\left\langle \nabla f(\theta),\, d \right\rangle
+
\int_0^1 \left\langle \nabla f(\theta + t d)-\nabla f(\theta),\, d \right\rangle dt.
\]
Using Cauchy--Schwarz and $L$-smoothness (i.e., $\|\nabla f(u)-\nabla f(v)\|\le L\|u-v\|$ on the segment),
for each $t\in[0,1]$ we have
\[
\Big|\left\langle \nabla f(\theta + t d)-\nabla f(\theta),\, d \right\rangle\Big|
\le
\|\nabla f(\theta + t d)-\nabla f(\theta)\|\,\|d\|
\le
L\,t\,\|d\|^2.
\]
Therefore,
\[
\int_0^1 \left\langle \nabla f(\theta + t d)-\nabla f(\theta),\, d \right\rangle dt
\le
\int_0^1 L t \|d\|^2\,dt
=
\frac{L}{2}\|d\|^2,
\]
which gives
\[
f(\theta') \le f(\theta)+\langle\nabla f(\theta),\theta'-\theta\rangle+\frac{L}{2}\|\theta'-\theta\|^2.
\]
Similarly, using
$\left\langle \nabla f(\theta + t d)-\nabla f(\theta),\, d \right\rangle \ge - L t \|d\|^2$
yields
\[
f(\theta') \ge f(\theta)+\langle\nabla f(\theta),\theta'-\theta\rangle-\frac{L}{2}\|\theta'-\theta\|^2.
\]
\end{proof}


\begin{theorem}[One-step decrease of $\mathcal L$ on a dominating sample]
\label{thm:app_component_jump}
Fix a step $i$ and an index $j^\star\in\mathcal N$.
Assume that at this realized step the group update is dominated by $j^\star$ in the sense that
\begin{equation}
\label{eq:app_dom_decomp_star}
\widehat g_{\mathrm{GRPO},i}
=
-\frac{1}{G}\,\hat\rho_{i,j^\star}\,|\widehat A_{j^\star}|\,h_{i,j^\star} + r_i^\star,
\qquad
\|r_i^\star\|\le \lambda\,\frac{1}{G}\,\hat\rho_{i,j^\star}\,|\widehat A_{j^\star}|\,\|h_{i,j^\star}\|.
\end{equation}
Define
\[
v:= h_{i,j^\star}=\nabla_\theta \mathcal L_{\theta_i}(x_{j^\star}),
\qquad
\eta := \frac{\eta_0}{G}\,\hat\rho_{i,j^\star}\,|\widehat A_{j^\star}|,
\qquad
\delta := \eta_0 r_i^\star.
\]
Then $\theta_{i+1}=\theta_i-\eta v+\delta$ and $\|\delta\|\le \lambda\eta\|v\|$.
Let $f_\star(\theta):=\mathcal L_\theta(x_{j^\star})$.
Assume $f_\star$ is $L_\star$-smooth on the realized segment $[\theta_i,\theta_{i+1}]$.
Then
\begin{equation}
\label{eq:app_component_jump_bound}
\mathcal L_{\theta_i}(x_{j^\star})-\mathcal L_{\theta_{i+1}}(x_{j^\star})
\ge
\eta\|v\|^2\left((1-\lambda)-\frac{L_\star}{2}(1+\lambda)^2\eta\right).
\end{equation}
In particular, if $\eta\le \frac{1-\lambda}{L_\star(1+\lambda)^2}$, then
\begin{equation}
\label{eq:app_component_jump_simplified}
\mathcal L_{\theta_i}(x_{j^\star})-\mathcal L_{\theta_{i+1}}(x_{j^\star})
\ge
\frac{1-\lambda}{2}\,\eta\,\|v\|^2.
\end{equation}
\end{theorem}

\begin{proof}
Apply Lemma~\ref{lem:app_smooth_quad} to $f_\star$ at $(\theta,\theta')=(\theta_i,\theta_{i+1})$:
\[
f_\star(\theta_{i+1}) \le f_\star(\theta_i)+\langle\nabla f_\star(\theta_i),\theta_{i+1}-\theta_i\rangle+\frac{L_\star}{2}\|\theta_{i+1}-\theta_i\|^2.
\]
Rearrange:
\[
f_\star(\theta_i)-f_\star(\theta_{i+1})
\ge
-\langle v,-\eta v+\delta\rangle-\frac{L_\star}{2}\|-\eta v+\delta\|^2
=
\eta\|v\|^2-\langle v,\delta\rangle-\frac{L_\star}{2}\|-\eta v+\delta\|^2.
\]
Bound $\langle v,\delta\rangle\le\|v\|\|\delta\|\le\lambda\eta\|v\|^2$.
Also $\|-\eta v+\delta\|\le \eta\|v\|+\|\delta\|\le (1+\lambda)\eta\|v\|$.
Substitute to obtain \eqref{eq:app_component_jump_bound}.
If $\eta\le \frac{1-\lambda}{L_\star(1+\lambda)^2}$, then the bracket is at least $(1-\lambda)/2$, yielding \eqref{eq:app_component_jump_simplified}.
\end{proof}


\begin{definition}[Anti-alignment and directional curvature]
\label{def:app_antialign}
Fix distinct indices $j^\star\neq j^\diamond$ and define
$x^\star:= x_{j^\star}$ and $x^\diamond:= x_{j^\diamond}$.
Let
\[
v:= \nabla_\theta \mathcal L_{\theta_i}(x^\star),
\qquad
u:= \nabla_\theta \mathcal L_{\theta_i}(x^\diamond),
\qquad
\gamma:= -\langle u,v\rangle.
\]
\end{definition}

\begin{theorem}[Cross-sample amplification with residual (proxy drift increase)]
\label{thm:app_cross_sample_amp_fixed}
Fix a step $i$ and two indices $j^\star\neq j^\diamond$ in $\mathcal N$.
Assume the outlier-dominance decomposition \eqref{eq:app_dom_decomp_star} holds on this realized step, and define
\begin{gather*}
v=\nabla_\theta \mathcal L_{\theta_i}(x_{j^\star}),
\qquad
u=\nabla_\theta \mathcal L_{\theta_i}(x_{j^\diamond}),
\qquad
\gamma=-\langle u,v\rangle>0, 
\\
\eta=\frac{\eta_0}{G}\hat\rho_{i,j^\star}\,|\widehat A_{j^\star}|,
\qquad
\theta_{i+1}=\theta_i-\eta v+\delta,
\qquad
\|\delta\|\le \lambda\eta\|v\|.
\end{gather*}
Let $f_\diamond(\theta):=\mathcal L_\theta(x_{j^\diamond})$ and assume $f_\diamond$ is $L_\diamond$-smooth on the realized segment
$[\theta_i,\theta_{i+1}]$.
If
\begin{equation}
\label{eq:app_amp_stepsize_cond_fixed}
\eta\le \frac{\gamma}{(1+\lambda)^2\,L_\diamond\|v\|^2},
\end{equation}
then on this realized step we have
\begin{equation}
\label{eq:app_amp_bound_fixed}
\mathcal L_{\theta_{i+1}}(x_{j^\diamond})-\mathcal L_{\theta_i}(x_{j^\diamond})
\ge
\eta\Big(\frac{\gamma}{2}-\lambda\|u\|\|v\|\Big),
\end{equation}
and consequently,
\begin{equation}
\label{eq:app_amp_drift_bound_fixed}
\Delta\mathcal L_{i+1,j^\diamond}
\ge
\Delta\mathcal L_{i,j^\diamond}
+\eta\Big(\frac{\gamma}{2}-\lambda\|u\|\|v\|\Big).
\end{equation}
\end{theorem}

\begin{proof}
Apply Lemma~\ref{lem:app_smooth_quad} (lower bound form) to $f_\diamond$ at $(\theta,\theta')=(\theta_i,\theta_{i+1})$:
\[
f_\diamond(\theta_{i+1})
\ge
f_\diamond(\theta_i)+\langle\nabla f_\diamond(\theta_i),\theta_{i+1}-\theta_i\rangle-\frac{L_\diamond}{2}\|\theta_{i+1}-\theta_i\|^2.
\]
Substitute $\nabla f_\diamond(\theta_i)=u$ and $\theta_{i+1}-\theta_i=-\eta v+\delta$:
\[
f_\diamond(\theta_{i+1})-f_\diamond(\theta_i)
\ge
\langle u,-\eta v+\delta\rangle-\frac{L_\diamond}{2}\|-\eta v+\delta\|^2
=
\eta\gamma+\langle u,\delta\rangle-\frac{L_\diamond}{2}\|-\eta v+\delta\|^2.
\]
Use $\langle u,\delta\rangle\ge -\|u\|\|\delta\|\ge -\lambda\eta\|u\|\|v\|$ and
$\|-\eta v+\delta\|\le (1+\lambda)\eta\|v\|$ to get
\[
f_\diamond(\theta_{i+1})-f_\diamond(\theta_i)
\ge
\eta\gamma-\lambda\eta\|u\|\|v\|-\frac{L_\diamond}{2}(1+\lambda)^2\eta^2\|v\|^2.
\]
Under \eqref{eq:app_amp_stepsize_cond_fixed}, the quadratic term is at most $\frac12\eta\gamma$, yielding \eqref{eq:app_amp_bound_fixed}.
Equation \eqref{eq:app_amp_drift_bound_fixed} is just rewriting in terms of $\Delta\mathcal L$.
\end{proof}

\begin{lemma}[From amplification of one sample to an increase in $D_i$]
\label{lem:app_D_update}
Fix a step $i$ and suppose Theorem~\ref{thm:app_cross_sample_amp_fixed} applies for some $j^\star\neq j^\diamond$ in $\mathcal N$.
Define
\[
\eta=\frac{\eta_0}{G}\hat\rho_{i,j^\star}\,|\widehat A_{j^\star}|.
\]
Define the $\mathcal F_{i-1}$-measurable coefficient
\[
c_{\mathrm{amp},i}
:=
\frac{\eta_0|\widehat A_{j^\star}|}{G}
\Big(\frac{\gamma}{2}-\lambda\|u\|\|v\|\Big)
\in\mathbb{R},
\]
where $u=\nabla_\theta \mathcal L_{\theta_i}(x_{j^\diamond})$,
$v=\nabla_\theta \mathcal L_{\theta_i}(x_{j^\star})$, and $\gamma=-\langle u,v\rangle>0$.
Then on this realized step,
\begin{equation}
\label{eq:app_D_increment}
D_{i+1}
\ge
D_i + \big(c_{\mathrm{amp},i}\hat\rho_{i,j^\star} - S_i\big).
\end{equation}
\end{lemma}

\begin{proof}
By Theorem~\ref{thm:app_cross_sample_amp_fixed},
\[
\Delta\mathcal L_{i+1,j^\diamond}
\ge
\Delta\mathcal L_{i,j^\diamond}
+\eta\Big(\frac{\gamma}{2}-\lambda\|u\|\|v\|\Big)
=
\Delta\mathcal L_{i,j^\diamond}+c_{\mathrm{amp},i}\hat\rho_{i,j^\star}.
\]
Since $D_{i+1}=\max_{j\in\mathcal N}\Delta\mathcal L_{i+1,j}\ge \Delta\mathcal L_{i+1,j^\diamond}$, we have
\[
D_{i+1}\ge \Delta\mathcal L_{i,j^\diamond}+c_{\mathrm{amp},i}\hat\rho_{i,j^\star}.
\]
By definition of $S_i$ in \eqref{eq:app_Di_Si_def},
\[
\Delta\mathcal L_{i,j^\diamond}
\ge
\min_{j\in\mathcal N}\Delta\mathcal L_{i,j}
=
D_i-S_i.
\]
Substituting yields \eqref{eq:app_D_increment}.
\end{proof}

\begin{proof}[Proof of Theorem~\ref{thm:grpo_instability_loop_core}]
The spike-probability bound \eqref{eq:core_spike_prob_main_new} follows from Lemma~\ref{lem:app_spike_prob_lower_new} and (C2):
\[
\mathbb{P}\!\big(\|\widehat g_{\mathrm{GRPO},i}\|\ge H \,\big|\, \mathcal F_{i-1}\big)
\ge
\frac{1}{2}\,\bar F(\log u_H-D_i).
\]
Monotonicity in $D_i$ holds by Lemma~\ref{lem:app_ratio_exceed}.

For the one-step decrease bound on a dominating sample, consider a realized step $i$ where a negative-advantage sample
$j^\star\in\mathcal N$ with $\hat\rho_{i,j^\star}\ge 1+\epsilon$ dominates the group update in the sense of
\eqref{eq:app_dom_decomp_star} and where the local smoothness/step-size condition of
Theorem~\ref{thm:app_component_jump} holds, including \eqref{eq:app_component_jump_simplified}.
Then \eqref{eq:app_component_jump_simplified} gives
\[
\mathcal L_{\theta_i}(x_{j^\star})-\mathcal L_{\theta_{i+1}}(x_{j^\star})
\ge
\frac{1-\lambda}{2}\,\eta\,\|v\|^2
=
\frac{1-\lambda}{2}\,\frac{\eta_0}{G}\,\hat\rho_{i,j^\star}|\widehat A_{j^\star}|\,\|v\|^2,
\]
with $v=\nabla_\theta \mathcal L_{\theta_i}(x_{j^\star})$.
Thus the first inequality in \eqref{eq:core_drift_creation_main} holds with
\[
c_{\mathrm{sup},i}
:=
\frac{1-\lambda}{2}\,\eta_0\,|\widehat A_{j^\star}|\,\|v\|^2
>0.
\]

For the drift-state increment, Lemma~\ref{lem:app_D_update} gives
\[
D_{i+1}\ge D_i+\big(c_{\mathrm{amp},i}\hat\rho_{i,j^\star}-S_i\big),
\]
establishing the second inequality in \eqref{eq:core_drift_creation_main}.

Finally, if $c_{\mathrm{amp},i}\hat\rho_{i,j^\star}\ge S_i$ then $D_{i+1}\ge D_i$.
Since $D\mapsto \bar F(\log u_H-D)$ is nondecreasing (Lemma~\ref{lem:app_ratio_exceed}),
the lower bound $P_i(H)=\frac12\bar F(\log u_H-D_i)$ cannot decrease from step $i$ to step $i{+}1$ on that realized step,
i.e., \eqref{eq:Pi_monotone_formal} holds.
\end{proof}

\subsection{Proof of Theorem~\ref{thm:uclip_instability_loop}}
\label{app:uclip_instability_loop_proof}

We prove Theorem~\ref{thm:uclip_instability_loop} for the \emph{two-sided} unconditional clipping rule
$w_{i,j}=\mathrm{clip}(\hat\rho_{i,j},\,1-\epsilon,\,1+\epsilon)$ and
$\widehat g_{\mathrm{clip},i}:=\frac{1}{G}\sum_{j=1}^G w_{i,j}g_{i,j}$.

\begin{lemma}[A sufficient upper-bound dominance event under two-sided clipping]
\label{lem:ub_dom_clipped}
Fix an inner step $i$ and let $j^\dagger\in\arg\max_{j\in\mathcal N}\Delta\mathcal L_{i,j}$.
Assume (C1) and (C4).
Define the residual
\[
r_i^{\mathrm{clip}} := \frac{1}{G}\sum_{j\neq j^\dagger} w_{i,j}g_{i,j}.
\]
On any realized step where $\hat\rho_{i,j^\dagger}\ge 1+\epsilon$ and
\begin{equation}
\label{eq:ub_dom_event}
\sum_{j\neq j^\dagger} w_{i,j} \ \le\ \frac{\lambda a_0 b_0}{B}\,(1+\epsilon),
\end{equation}
we have the deterministic decomposition
\[
\widehat g_{\mathrm{clip},i}
=
-\frac{1}{G}(1+\epsilon)\,|\widehat A_{j^\dagger}|\,h_{i,j^\dagger} + r_i^{\mathrm{clip}},
\qquad
\|r_i^{\mathrm{clip}}\|
\le
\lambda\,\frac{1}{G}(1+\epsilon)\,|\widehat A_{j^\dagger}|\,\|h_{i,j^\dagger}\|.
\]
\end{lemma}

\begin{proof}
On $\hat\rho_{i,j^\dagger}\ge 1+\epsilon$, we have $w_{i,j^\dagger}=1+\epsilon$ and
$g_{i,j^\dagger}=\widehat A_{j^\dagger}h_{i,j^\dagger}=-|\widehat A_{j^\dagger}|h_{i,j^\dagger}$.
Thus
\[
\widehat g_{\mathrm{clip},i}
=
\frac{1}{G}w_{i,j^\dagger}g_{i,j^\dagger}
+
\frac{1}{G}\sum_{j\neq j^\dagger} w_{i,j}g_{i,j}
=
-\frac{1}{G}(1+\epsilon)|\widehat A_{j^\dagger}|h_{i,j^\dagger}
+
r_i^{\mathrm{clip}}.
\]
Moreover, by (C1),
\[
\|r_i^{\mathrm{clip}}\|
\le
\frac{1}{G}\sum_{j\neq j^\dagger} w_{i,j}\|g_{i,j}\|
\le
\frac{B}{G}\sum_{j\neq j^\dagger} w_{i,j}.
\]
Under \eqref{eq:ub_dom_event}, this yields
\[
\|r_i^{\mathrm{clip}}\|
\le
\frac{B}{G}\cdot \frac{\lambda a_0 b_0}{B}(1+\epsilon)
=
\lambda\,\frac{1}{G}(1+\epsilon)a_0 b_0
\le
\lambda\,\frac{1}{G}(1+\epsilon)|\widehat A_{j^\dagger}|\|h_{i,j^\dagger}\|,
\]
since $|\widehat A_{j^\dagger}|\ge a_0$ and $\|h_{i,j^\dagger}\|\ge b_0$ by (C4).
\end{proof}

\begin{proof}[Proof of Theorem~\ref{thm:uclip_instability_loop}]
Fix an inner step $i$.
By (C1) we have $\|g_{i,j}\|\le B$ for all $j$.
Moreover, since $\hat\rho_{i,j}>0$ and $w_{i,j}=\mathrm{clip}(\hat\rho_{i,j},1-\epsilon,1+\epsilon)$,
we have $0<w_{i,j}\le 1+\epsilon$. Therefore,
\[
\|\widehat g_{\mathrm{clip},i}\|
=
\Big\|\frac{1}{G}\sum_{j=1}^G w_{i,j}\,g_{i,j}\Big\|
\le
\frac{1}{G}\sum_{j=1}^G w_{i,j}\,\|g_{i,j}\|
\le
(1+\epsilon)B,
\]
which proves the deterministic boundedness claim.

Let $j^\dagger\in\arg\max_{j\in\mathcal N}\Delta\mathcal L_{i,j}$ so that $\Delta\mathcal L_{i,j^\dagger}=D_i$.
By Lemma~\ref{lem:app_ratio_exceed} with $u=1+\epsilon$ and (C2),
\[
\mathbb{P}\!\big(\hat\rho_{i,j^\dagger}\ge 1+\epsilon \,\big|\, \mathcal F_{i-1}\big)
=
\bar F_{j^\dagger,i}\!\big(\log(1+\epsilon)-D_i\big)
\ge
\bar F\!\big(\log(1+\epsilon)-D_i\big),
\]
and the right-hand side is nondecreasing in $D_i$ by Lemma~\ref{lem:app_ratio_exceed},
which establishes \eqref{eq:core_upper_sat}.

Finally, on any realized step where the sufficient dominance event in Lemma~\ref{lem:ub_dom_clipped} holds and where the
local smoothness/geometry conditions required by Theorem~\ref{thm:app_cross_sample_amp_fixed}
(with the effective step size $\eta=\frac{\eta_0}{G}(1+\epsilon)|\widehat A_{j^\dagger}|$) hold for some $j^\diamond\in\mathcal N\setminus\{j^\dagger\}$,
the same argument as Lemma~\ref{lem:app_D_update} yields
\[
D_{i+1}\ \ge\ D_i + \big(c_{\mathrm{amp},i}(1+\epsilon)-S_i\big).
\]
This completes the proof.
\end{proof}

\subsection{Proof of Theorem~\ref{thm:selfnorm_convex_hull}}
\label{app:selfnorm_convex_hull_proof}

\begin{proof}[Proof of Theorem~\ref{thm:selfnorm_convex_hull}]
Fix an inner step $i$ and define the two-sided clipped weights
\[
w_{i,j} := \mathrm{clip}(\hat\rho_{i,j},\,1-\epsilon,\,1+\epsilon),
\qquad j=1,\ldots,G.
\]
Since $\hat\rho_{i,j}>0$, we have $w_{i,j}>0$ and thus $\sum_{k=1}^G w_{i,k}>0$.
Define
\[
\widehat g_{\mathrm{sn},i}
:=
\frac{\sum_{j=1}^G w_{i,j}\,g_{i,j}}{\sum_{j=1}^G w_{i,j}}.
\]
Let $\alpha_{i,j} := w_{i,j}/\sum_{k=1}^G w_{i,k}$. Then $\alpha_{i,j}\ge 0$ and $\sum_{j=1}^G \alpha_{i,j}=1$, hence
\[
\widehat g_{\mathrm{sn},i}
=
\sum_{j=1}^G \alpha_{i,j} g_{i,j}
\in
\operatorname{conv}\{g_{i,1},\ldots,g_{i,G}\}.
\]
By (C1), $\|g_{i,j}\|\le B$ for all $j$, therefore
\[
\|\widehat g_{\mathrm{sn},i}\|
\le
\sum_{j=1}^G \alpha_{i,j} \|g_{i,j}\|
\le
\sum_{j=1}^G \alpha_{i,j} B
=
B.
\]
This proves the deterministic bound and the convex-hull property.
\end{proof}

\newpage
\section{Experimental Details}
\label{app:exp}

\subsection{Training and Hyperparameter Setup}
\label{app:training_setup}

We provide detailed configurations for our experiments on both Full-Attention Diffusion and Block Diffusion architectures to ensure reproducibility. All experiments were conducted using the StableDRL framework, with hyperparameters chosen to isolate the contribution of our stability mechanisms.

\subsubsection{Full-Attention Diffusion (LLaDA-8B-Instruct)}

We fine-tune the \texttt{LLaDA-8B-Instruct} model using iterative decoding with a generation length of 256 tokens and a block size of 32. Optimization is performed using AdamW with a learning rate of $1.0 \times 10^{-6}$ and a linear decay schedule over 2,000 steps. Crucially, we enable Self-Normalized Importance Sampling (SNIS) with an unconditional importance weight clipping threshold of 5.0.

Table~\ref{tab:hyperparams_full_attn} summarizes the complete hyperparameter configuration.

\begin{table}[h]
    \centering
    \caption{Hyperparameter Configuration for Full-Attention Diffusion (LLaDA-8B-Instruct)}
    \label{tab:hyperparams_full_attn}
    \begin{tabular}{l|l}
        \toprule
        \textbf{Category} & \textbf{Value} \\
        \midrule
        \multicolumn{2}{l}{\textit{Model \& Initialization}} \\
        Base Model & LLaDA-8B-Instruct \\
        Precision & bfloat16 \\
        Activation Checkpointing & Whole Layer \\
        \midrule
        \multicolumn{2}{l}{\textit{Generation (Rollout)}} \\
        Decoding Strategy & Iterative (128 steps) \\
        Generation Length & 256 tokens \\
        Block Size & 32 \\
        Temperature & 0.9 \\
        Rollout Scale & 8 generations $\times$ 2 repeats \\
        \midrule
        \multicolumn{2}{l}{\textit{Training \& Optimization}} \\
        Optimizer & AdamW ($\beta_1=0.9, \beta_2=0.99, \lambda=0.1$) \\
        Learning Rate & $1.0 \times 10^{-6}$ (Linear Decay) \\
        Batch Size & 1 per GPU (Grad Accumulation = 4) \\
        Gradient Clipping & 0.2 \\
        Inner Updates & 2 per rollout cycle \\
        Total Steps & 2000 \\
        \midrule
        \multicolumn{2}{l}{\textit{StableDRL Specifics}} \\
        Loss Function & Sandwiched ($\beta=1.5, \omega=0.5$) \\
        ELBO Estimation & 2 MC samples (perturbation $p=0.15$) \\
        Stabilization & SN enabled, Clip Threshold = 5.0 \\
        \bottomrule
    \end{tabular}
\end{table}

\subsubsection{Block Diffusion (SDAR-8B-Chat)}

We instantiate StableDRL on the \texttt{SDAR-8B-Chat} architecture, following the conventions of TraceRL extended with our stability mechanisms. We utilize dynamic sampling with a threshold of $\tau=0.9$ and a temperature of 1.0. The model is trained using AdamW with a learning rate of $1.0 \times 10^{-6}$ and no weight decay. To stabilize the group-wise updates, we employ Group-wise SNIS with an asymmetric log-clipping threshold of 5.0 (log-space). We also enable mask resampling in the trainer to maintain valid drift coupling during optimization.

Table~\ref{tab:hyperparams_block_diff} details the configuration for the block diffusion experiments.

\begin{table}[h]
    \centering
    \caption{Hyperparameter Configuration for Block Diffusion (SDAR-8B-Chat)}
    \label{tab:hyperparams_block_diff}
    \begin{tabular}{l|l}
        \toprule
        \textbf{Category} & \textbf{Value} \\
        \midrule
        \multicolumn{2}{l}{\textit{Model \& Initialization}} \\
        Base Model & JetLM/SDAR-8B-Chat \\
        Architecture & Block Diffusion ($B=4$) \\
        Precision & bf16 (TF32 enabled) \\
        \midrule
        \multicolumn{2}{l}{\textit{Generation (Rollout)}} \\
        Sampling Strategy & Dynamic ($\tau=0.9, T=1.0$) \\
        Denoising Steps & 4 per block \\
        Rollout Scale & 16 responses per task \\
        \midrule
        \multicolumn{2}{l}{\textit{Training \& Optimization}} \\
        Optimizer & AdamW ($lr=1\text{e-}6, \beta_2=0.999$, no decay) \\
        Scheduler & Linear Decay \\
        Micro Batch Size & 1 (Gradient Accumulation = 2) \\
        Gradient Clipping & 1.0 \\
        \midrule
        \multicolumn{2}{l}{\textit{StableDRL Specifics}} \\
        Advantage Mode & Raw Centered \\
        Importance Sampling & Group-wise SNI \\
        Clip Threshold & 5.0 (log-space) \\
        Mask Resampling & Enabled \\
        \bottomrule
    \end{tabular}
\end{table}

\subsection{Details of the Exploding Importance Weight Protocol}
\label{app:stress_test}

To validate the robustness of StableDRL against the heavy-tailed noise characteristic of dLLMs, we use a controlled adversarial protocol that artificially inflates the variance of the importance ratio $\hat{\rho}$.

\subsection{Mechanism: Asymmetric Masking}
The importance ratio is estimated as $\hat{\rho} = \exp(\hat{\mathcal{L}}_\theta - \hat{\mathcal{L}}_{\text{old}})$. We induce "exploding" weights by breaking the symmetry of the Monte Carlo estimation for a random 70\% subset of the batch (the "stressed" samples). We employ two decoupled masking policies:
\begin{enumerate}
    \item \textbf{Numerator ($\hat{\mathcal{L}}_\theta$) $\rightarrow$ "Easy" Masking:} We bias masking towards high-confidence regions (e.g., the sequence tail) and select the \emph{minimum} number of masked tokens ($t_{\min}$). This yields a tighter, optimistic ELBO estimate.
    \item \textbf{Denominator ($\hat{\mathcal{L}}_{\text{old}}$) $\rightarrow$ "Hard" Masking:} We bias masking towards low-confidence regions (e.g., the sequence head) and select the \emph{maximum} number of masked tokens ($t_{\max}$). This yields a looser, pessimistic ELBO estimate.
\end{enumerate}
This systematic gap ensures that $\hat{\mathcal{L}}_\theta \gg \hat{\mathcal{L}}_{\text{old}}$, driving $\hat{\rho} \to \infty$ purely due to estimation variance, independent of the actual policy probability.

\subsection{Implementation}
We operationalize "Easy" vs. "Hard" based on the diffusion formulation (Block vs. Random Token). Algorithm~\ref{alg:exploding_iw} details the generation process.

\begin{algorithm}[h]
\caption{Adversarial Generation of Exploding Importance Weights}
\label{alg:exploding_iw}
\begin{algorithmic}[1]
\REQUIRE Batch $X$, Group size $G$, Coverage fraction $\gamma=0.7$
\REQUIRE Bias Strength $\beta=6.0$ (for Random), Masking Policy $\mathcal{P} \in \{\text{Block}, \text{Random}\}$
\FOR{each group $g$ in Batch}
    \STATE Select indices $S_g \subset g$ with size $\lceil \gamma \cdot G \rceil$ to stress.
    \FOR{each sample $x_i$ in group $g$}
        \IF{$i \in S_g$}
            \STATE \textcolor{gray}{// 1. Numerator: "Easy" Masking (Tail Bias + Min Count)}
            \IF{$\mathcal{P}$ is Block}
                \STATE $M_{\text{num}} \gets$ Mask indices of the \textbf{Last Block} (Max Context)
            \ELSE
                \STATE $W[k] \propto \exp(+\beta \cdot k/L)$ \COMMENT{Tail Position Bias}
                \STATE $M_{\text{num}} \sim \text{Multinomial}(W)$
                \STATE $\text{Count}(M_{\text{num}}) \gets t_{\min}$ \COMMENT{Min Masked Tokens}
            \ENDIF
            
            \STATE \textcolor{gray}{// 2. Denominator: "Hard" Masking (Head Bias + Max Count)}
            \IF{$\mathcal{P}$ is Block}
                \STATE $M_{\text{den}} \gets$ Mask indices of the \textbf{First Block} (Min Context)
            \ELSE
                \STATE $W[k] \propto \exp(-\beta \cdot k/L)$ \COMMENT{Head Position Bias}
                \STATE $M_{\text{den}} \sim \text{Multinomial}(W)$
                \STATE $\text{Count}(M_{\text{den}}) \gets t_{\max}$ \COMMENT{Max Masked Tokens}
            \ENDIF
        \ELSE
            \STATE \textcolor{gray}{// Control: Standard Uniform Masking}
            \STATE $M_{\text{num}}, M_{\text{den}} \sim \text{UniformRandom}(x_i)$
        \ENDIF
        
        \STATE $\hat{\mathcal{L}}_\theta \gets \text{ComputeELBO}(x_i, \pi_\theta, M_{\text{num}})$
        \STATE $\hat{\mathcal{L}}_{\text{old}} \gets \text{ComputeELBO}(x_i, \pi_{\text{old}}, M_{\text{den}})$
        \STATE $\hat{\rho}_i \gets \exp(\hat{\mathcal{L}}_\theta - \hat{\mathcal{L}}_{\text{old}})$
    \ENDFOR
\ENDFOR
\STATE \textbf{return} Importance Weights $\hat{\rho}$
\end{algorithmic}
\end{algorithm}

\subsection{Visual Diagnosis of Gradient Instability}
\label{app:gradient_diagnosis}

To empirically validate the ``Instability Feedback Loop'' and the structural failures diagnosed in Section 3.1, we visualize the joint distribution of importance weights ($\log_{10} \rho$) and gradient norms ($\log_{10} \|\hat{g}\|$) recorded during training. Figure~\ref{fig:gradient_instability} presents a comparative diagnostic of ESPO, SPG-IS, and StableDRL, offering a direct geometric validation of our theoretical analysis.

\textbf{The ``Chimney'' Failure in ESPO.} As observed in the left panel, ESPO exhibits a pathological ``chimney'' distribution. While the majority of samples cluster in a low-variance region, a sparse subset of noise-induced outliers (importance weights $\rho > 10^6$) drives gradient norms to catastrophic levels ($\|\hat{g}\| > 10^4$). This empirically confirms \textit{Failure 1 (Asymmetric Failure of the Clipped Surrogate)} described in Section 3.1: when a sample with a large noise-induced importance weight has a negative advantage, it falls into the unclipped branch of the objective. Consequently, these ``trapdoor'' outliers bypass the trust region and act as unbounded multipliers on the step size, injecting massive shocks that destabilize the policy.

\textbf{Drift-Variance Correlation in SPG-IS.} The center panel displays the dynamics of SPG-IS. Although SPG avoids explicit ratio computation to mitigate the ``chimney'' effect, the visualization reveals a strong positive correlation between the implicit weight magnitude and the update norm. This indicates that the method remains sensitive to policy drift: as the target policy diverges from the behavior policy, the accumulated ``rollout-reuse bias'' scales the variance of the updates proportionally. This prevents convergence, as the method lacks the structural constraints to mechanically decouple the update magnitude from distribution shifts.

\textbf{Geometric Stability in StableDRL.} In contrast, the right panel demonstrates the efficacy of our proposed framework. StableDRL displays a compact, bounded distribution where gradient norms remain consistently low ($< 10^{1.8}$) regardless of the importance weight magnitude. This confirms the effect of our dual stability mechanisms: \textit{Unconditional Clipping} strictly censors extreme ratios before aggregation, while \textit{Self-Normalization} ensures the update remains a convex combination of per-sample gradients. As predicted by \textit{Theorem 3.1}, StableDRL effectively confines the update to the convex hull of the samples, maintaining deterministic stability even in the presence of heavy-tailed proxy noise.

\begin{figure}[h]
    \centering
    \includegraphics[width=\linewidth]{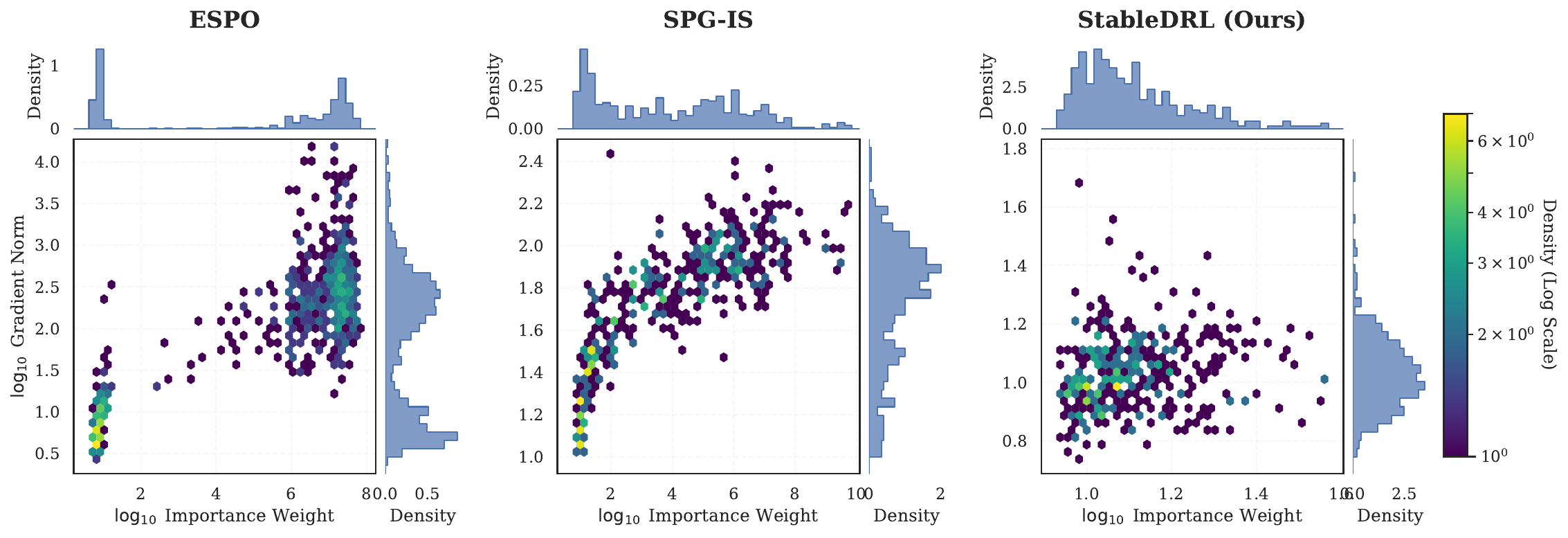}
    \caption{\textbf{Diagnosing Gradient Instability in dLLM Training.} We visualize the joint distribution of importance weights ($\log_{10} \rho$) and gradient norms ($\log_{10} \|\hat{g}\|$) during training. \textbf{(Left) ESPO:} Exhibits a characteristic ``chimney'' failure where rare, noise-induced outliers bypass clipping on negative advantages, acting as unbounded step-size multipliers that drive gradients to explosion ($> 10^4$). \textbf{(Center) SPG-IS:} Despite avoiding explicit ratios, the update variance is strongly correlated with policy drift, confirming that rollout-reuse bias accumulates to destabilize training. \textbf{(Right) StableDRL (Ours):} By enforcing strict clipping and self-normalization, our method decouples update magnitude from proxy noise, confining gradients to the convex hull of the samples (Theorem 3.1) and maintaining deterministic stability.}
    \label{fig:gradient_instability}
\end{figure}

\end{document}